\documentclass{article}

\usepackage[preprint]{neurips_2026}

\usepackage[utf8]{inputenc} 
\usepackage[T1]{fontenc}    

\usepackage{amsmath,amssymb,amsfonts}
\usepackage{amsbsy}
\usepackage{amsthm}
\usepackage{graphicx}
\usepackage{textcomp}
\usepackage{xcolor}
\usepackage{hyperref}
\usepackage{url}
\usepackage{booktabs}
\usepackage{nicefrac}
\usepackage{microtype}
\usepackage{algorithmic}
\usepackage[]{algorithm}
\usepackage{enumerate}
\usepackage{caption}
\usepackage{subcaption}
\usepackage{multirow}
\usepackage{tabu}
\usepackage[normalem]{ulem}
\usepackage{adjustbox}
\usepackage{placeins}

\definecolor{citeblue}{RGB}{25,80,150}

\hypersetup{
    colorlinks=true,
    citecolor=citeblue,
    linkcolor=black,
    urlcolor=citeblue
}

\def\1{{\bf 1}}

\def\be{\begin{equation}}
\def\ee{\end{equation}}

\DeclareMathOperator*{\argmin}{arg\,min} 

\theoremstyle{plain}
\newtheorem{theorem}{Theorem}[section]

\newtheorem{lemma}[theorem]{Lemma}

\theoremstyle{definition}

\newtheorem{assumption}[theorem]{Assumption}

\theoremstyle{remark}
\newtheorem{remark}[theorem]{Remark}

\DeclareMathOperator*{\argmax}{arg\,max}

\title{Distance-Aware Muon: Adaptive Step Scaling for Normalized Optimization}

\author{%
  Yury Demidovich\thanks{Equal contribution.} \\
  King Abdullah University of Science and Technology \\
  \texttt{yury.demidovich@kaust.edu.sa} \\
  \And
  Abhishek Chakraborty\footnotemark[1] \\
  Arizona State University \\
  \texttt{achakr61@asu.edu} \\
  \And
  Grigory Malinovsky \\
  King Abdullah University of Science and Technology \\
  \texttt{grigorii.malinovskii@kaust.edu.sa} \\
  \And
  Angelia Nedi{\'c} \\
  Arizona State University \\
  \texttt{Angelia.Nedich@asu.edu} \\
  \And
  Peter Richt\'arik \\
  King Abdullah University of Science and Technology \\
  \texttt{peter.richtarik@kaust.edu.sa} \\
}

\begin{document}

\maketitle

\begin{abstract} Muon and related normalized optimizers decouple the choice of update direction from the choice of step scale, but their practical performance remains sensitive to the scale of the normalized step. We study adaptive scaling rules for Muon in general norm geometries and develop three complementary algorithms. For smooth non-convex objectives, we introduce Distance-Adaptive Muon, whose trust-region radius is set from the radius explored by the trajectory, and prove a stationarity guarantee under a bounded-trajectory assumption. We then turn to star-convex objectives, a tractable model of the favorable global geometry often used to reason about the empirical loss landscapes of deep neural networks, where objective-gap guarantees are possible. In this setting, we first introduce Scale-Calibrated Muon, which keeps Muon's exponential moving average but sets the step length from a local descent certificate computed from the current gradient and momentum. For this method, we prove a last-iterate \(O(1/T)\) objective-gap bound under a bounded initial sublevel-set assumption, where the corresponding radius parameter appears only in the analysis and not in the algorithm. Finally, we develop Distance-Free Muon, a recentered trust-region method that uses a scalar distance certificate and a majorized one-dimensional search to select the trust-region radius without requiring the unknown distance from the initialization to a global minimizer. Experiments on Transformer language modeling (GPT-124M/WikiText-103) and image classification (ViT-Tiny/CIFAR-100) show that the proposed adaptive scaling rules reduce sensitivity to manual scale tuning and match or improve tuned fixed-scale Muon baselines under the tested budgets.
\end{abstract}

\section{Introduction}

Choosing the scalar step size remains a central practical difficulty in training modern machine learning models. Adaptive methods such as AdaGrad, Adam, and AdamW reduce sensitivity to poorly scaled coordinates and are standard tools across architectures and datasets~\citep{duchi2011adaptive,kingma2015adam,loshchilov2019decoupled}. Yet even these methods retain learning-rate and schedule parameters whose best values depend on the problem, model, and training budget. Poor scale choices can make training unstable or overly conservative, and tuning them can dominate the cost of large-scale training runs~\citep{bengio2000gradient,feurer2019hyperparameter}.

Muon and related normalized optimizers use a geometry-dependent normalized update direction rather than the raw gradient direction. In particular, methods based on gradient orthogonalization have shown strong empirical performance in neural-network training~\citep{jordan2024muon,large2024scalable,liu2025muon}. Their key idea is to decouple the direction of the update from its scalar magnitude. Muon forms a normalized direction adapted to the geometry of matrix-valued parameters. The non-Euclidean trust-region interpretation of \citet{kovalev2025understanding} makes this separation precise: the norm geometry determines a linear-minimization-oracle direction, while the scalar trust-region radius remains to be chosen. In practice, this scalar radius is tuned as a learning rate, often with warmup, decay, clipping, and architecture-specific heuristics. In theory, existing analyses typically choose the radius using problem-dependent quantities such as a domain diameter, a distance to a minimizer, or an initial optimality gap. These quantities are rarely known in the neural-network settings where Muon is used.

The problem of removing such hidden scale parameters has been studied in parameter-free and distance-adaptive optimization~\citep{orabona2016coin,cutkosky2018black,carmon2022making,ivgi2023dog,defazio2023learning,moshtaghifar2025dada}. These works show that scalar scale selection can often be driven by observable trajectory statistics rather than by hand-tuned constants. However, their mechanisms are designed primarily for raw-gradient, online-learning, or dual-averaging updates. They do not directly address the Muon setting, where the direction is produced by a norm-induced linear minimization oracle. The question we address is whether one can choose the scalar radius adaptively while keeping the geometry-aware Muon direction.

Beyond the general non-convex setting, we also study star-convex objectives. Star-convexity is a relaxation of convexity that preserves a first-order inequality toward a global minimizer and therefore allows objective-gap guarantees rather than only stationarity guarantees. This assumption is also motivated by studies of neural-network optimization, where star-convex paths, one-point convexity, and related variants have been proposed as models of favorable non-convex loss geometry~\citep{kleinberg2018alternative,zhou2019sgd,kovalev2025understanding}. In the closest existing Muon-type star-convex analysis, \citet{kovalev2025understanding} obtains convergence guarantees using radius scales linked to distance or localization quantities, such as a distance to a minimizer, a bounded domain, or a bounded sublevel set. We aim to remove this hidden radius input while keeping the Muon/LMO direction. We discuss broader connections to normalized optimization, parameter-free methods, distance-adaptive optimization, star-convexity, and regularized model methods in Section~\ref{sec:related_work}.

This paper studies adaptive scalar scaling for Muon-type normalized optimizers. We use the non-Euclidean trust-region view of Muon as the organizing principle: the norm geometry determines the update direction, and the scalar trust-region radius is the remaining quantity to choose. We develop three complementary scale rules. Each method preserves the Muon linear-minimization direction and changes only the scalar rule used to choose the radius. The three rules cover different regimes: a trajectory-radius rule for general non-convex stationarity, a local descent certificate for star-convex last-iterate convergence, and a majorized scalar search that removes bounded localization assumptions.

\paragraph{Contributions.}
Our main contributions are as follows.

\begin{enumerate}
    \item \textbf{Distance-Adaptive Muon for smooth non-convex objectives.} We introduce a trajectory-radius rule
    \[
        \bar r_{k+1}=\max\{\bar r_k,\|x_k-x_0\|\},
        \qquad
        \eta_k=\frac{\bar r_{k+1}}{\sqrt{k+1}}.
    \]
    The update uses the distance already explored by the iterates rather than a prescribed domain scale. For smooth non-convex objectives whose generated trajectory remains bounded, we prove an \(O(\log T/\sqrt T)\) stationarity guarantee.

    \item \textbf{Scale-Calibrated Muon under bounded initial sublevel sets.} We introduce a local certificate-based scaling rule for Muon momentum. Let \(m_{k+1}\) be the momentum vector, $L$ be the Lipschitz smoothness constant, and \(g_k=\nabla f(x_k)\). Then, the method sets
    \[
        a_k=
        \left(
        \|m_{k+1}\|_*-\|g_k-m_{k+1}\|_*
        \right)_+,
        \qquad
        \eta_k=\frac{a_k}{L}.
    \]
    This certificate reduces the step whenever the momentum direction is not aligned with the current gradient. Under star-convexity and boundedness of the initial sublevel set, we prove a last-iterate \(O(LD_{\rm lev}^2/T)\) objective-gap guarantee. The level-set radius \(D_{\rm lev}\) appears only in the analysis and is not used by the algorithm.

    \item \textbf{Distance-Free Muon for star-convex objectives.} We develop a recentered Muon method with a majorized scalar radius search. The method keeps the Muon linear-minimization direction with exponential momentum, but replaces the unknown radius choice \(\eta=\beta\|x^\star-x_0\|\) by minimizing a regularized smoothness majorant along the recentered ray. A D-adaptation-style scalar proxy \(d_k\) provides a monotone lower certificate for \(\|x^\star-x_0\|\), but the proof does not require this proxy to recover the true distance. Instead, the scalar search competes with the hidden feasible radius \(R=\|x^\star-x_0\|\). Under star-convexity and smoothness, we prove the last-iterate guarantee
    \[
        f(x_T)-f(x^\star)
        =
        \widetilde O\!\left(
            \frac{L\|x^\star-x_0\|^2}{T}
        \right),
    \]
    without assuming bounded iterates, bounded gradients, a bounded domain, or a bounded initial sublevel set. To the best of our knowledge, this is the first Muon-type star-convex guarantee that removes the distance-to-solution parameter without replacing it by one of these localization assumptions.

    \item \textbf{Empirical evaluation.} We evaluate the proposed scaling rules on Transformer language modeling with GPT-124M/WikiText-103 and image classification with regularized ViT-Tiny/CIFAR-100. The experiments test whether the deterministic scale-selection principles remain useful under stochastic gradients and standard neural-network stabilizations. Under the tested hyperparameter budgets, the adaptive Muon variants reduce sensitivity to manual scale tuning and match or improve tuned fixed-scale Muon baselines. Additional diagnostics on NanoGPT/WikiText-2 and CIFAR-100/ResNet-32 are reported in Appendix~\ref{app:additional_experiments}.
\end{enumerate}

These results support the view that Muon-type methods combine two separable design choices: a geometry-aware normalized direction and a scalar radius. Existing theory largely explains the former. This paper develops adaptive rules for the latter.
\section{Problem Setup and Assumptions}
\label{sec:setup}

We study the optimization problem
\begin{equation}
    \min_{x\in \mathcal E} f(x),
    \label{eq:main_problem}
\end{equation}
where \(\mathcal E\) is a finite-dimensional real vector space, possibly a matrix space or a product of matrix spaces, and \(f:\mathcal E\to\mathbb R\) is continuously differentiable. We write \(f^\star:=\inf_{x\in\mathcal E} f(x)\). In the sections where objective-gap guarantees are proved, we assume that the infimum is attained at a global minimizer \(x^\star\in\arg\min_{x\in\mathcal E} f(x)\).

The space \(\mathcal E\) is equipped with a norm \(\|\cdot\|\), and the corresponding dual norm is
\[
    \|g\|_*:=\sup_{\|u\|\le1}\langle g,u\rangle .
\]
Distances between iterates are measured in \(\|\cdot\|\), while gradients and momentum vectors are measured in \(\|\cdot\|_*\). This norm-generic formulation covers Euclidean geometry, sign-type geometries, and the matrix spectral--nuclear geometry underlying Muon; in the latter case, the resulting linear-minimization direction is the usual orthogonalized-gradient direction~\citep{kovalev2025understanding}.

Given a momentum vector \(m\), define the norm-induced direction
\[
    u(m)\in\argmax_{\|u\|\le1}\langle m,u\rangle,
    \qquad
    \langle m,u(m)\rangle=\|m\|_* .
\]

Equivalently, we sometimes write
\[
    s(m)\in\argmin_{\|s\|\le1}\langle m,s\rangle,
    \qquad
    s(m)=-u(m),
\]
so that a step of radius \(\eta_k\) can be written as \(x_{k+1}=z_k+\eta_k s(m_{k+1})\).

A Muon-type trust-region update has the form
\begin{equation}
    x_{k+1}=z_k-\eta_k u(m_{k+1}),
    \label{eq:generic_tr_step}
\end{equation}
where \(z_k\) is the center and \(\eta_k\ge0\) is the scalar radius. The standard Muon-type update uses \(z_k=x_k\). The algorithms in this paper keep the same norm-induced direction and differ only in how they choose \(z_k\) and \(\eta_k\).

We first state the smoothness assumption shared by the analyses. Additional assumptions are introduced locally in the sections where they are needed.

\begin{assumption}[Norm smoothness]
\label{assum:smoothness}
For all \(x,y\in\mathcal E\),
\[
    \|\nabla f(y)-\nabla f(x)\|_*
    \le L\|y-x\|.
\]
Consequently,
\[
    f(y)
    \le
    f(x)+\langle\nabla f(x),y-x\rangle+\frac L2\|y-x\|^2 .
\]
\end{assumption}

Smoothness provides the curvature control needed for normalized trust-region steps and for bounding momentum-tracking errors. Our goal is to remove unknown distance-to-solution and localization radii from the scalar scaling rules, not to remove the standard smoothness information used in deterministic first-order analyses.

For the objective-gap results, we use the following global geometry condition.

\begin{assumption}[Star-convexity]
\label{assum:star_convex}
There exists a global minimizer \(x^\star\in\arg\min_{x\in\mathcal E} f(x)\) such that,
\[
    f(x)-f(x^\star)
    \le
    \langle \nabla f(x),x-x^\star\rangle  \quad \text{for all \(x\in\mathcal E\)}.
\]
\end{assumption}

\paragraph{Assumption map.}
The three algorithms require different structural assumptions. For brevity, we refer to Distance-Adaptive Muon, Scale-Calibrated Muon, and Distance-Free Muon as DA-Muon, SC-Muon, and DF-Muon, respectively.

\begin{center}
\small
\setlength{\tabcolsep}{3.5pt}
\renewcommand{\arraystretch}{1.0}
\begin{tabular}{lll}
\toprule
\textbf{Method} & \textbf{Objective} & \textbf{Extra condition} \\
\midrule
DA-Muon
& smooth non-convex
& bounded trajectory \\

SC-Muon
& smooth star-convex
& bounded initial sublevel set \\

DF-Muon
& smooth star-convex
& majorized scalar search; no boundedness \\
\bottomrule
\end{tabular}
\end{center}

The bounded-trajectory assumption is not part of the global problem formulation; it is used only for the non-convex stationarity result. Similarly, star-convexity is not assumed in the non-convex section, and the bounded initial sublevel-set assumption is used only for the scale-calibrated last-iterate guarantee.
\section{Distance-Adaptive Muon for Non-convex Objectives}
\label{sec:da_muon}

We first instantiate the non-Euclidean trust-region viewpoint on Muon \citep{kovalev2025understanding} in the smooth non-convex setting. The natural stationarity measure is the full dual gradient norm \(\|\nabla f(x)\|_*\). Distance-Adaptive Muon uses the standard Muon center \(z_k=x_k\), but chooses the trust-region radius from the radius already explored by the trajectory.

Starting from an initial scale proxy \(r>0\), the method maintains
\[
    \bar r_{k+1}
    =
    \max\{\bar r_k,\|x_k-x_0\|\},
    \qquad
    \eta_k
    =
    \frac{\bar r_{k+1}}{\sqrt{k+1}}.
\]
Thus the step scale increases only when the iterates move farther from the initialization, while still retaining the usual \(1/\sqrt{k}\)-type decay.

\begin{algorithm}[H]
\caption{Distance-Adaptive Muon}
\label{algo_dis_adaptive}
\begin{algorithmic}[1]
\REQUIRE Initial point \(x_0\in\mathcal E\), \(m_0=\nabla f(x_0)\), \(\bar r_0=r>0\), and \(\alpha\in(0,1)\)
\FOR{\(k=0,1,\dots,T-1\)}
    \STATE \(\bar r_{k+1}\gets \max\{\bar r_k,\|x_k-x_0\|\}\)
    \STATE \(\eta_k\gets \bar r_{k+1}/\sqrt{k+1}\)
    \STATE \(m_{k+1}\gets (1-\alpha)m_k+\alpha\nabla f(x_k)\)
    \STATE Choose
    \[
        u_{k+1}
        \in
        \argmax_{\|u\|\le 1}
        \langle m_{k+1},u\rangle
    \]
    \STATE \(x_{k+1}\gets x_k-\eta_k u_{k+1}\)
\ENDFOR
\end{algorithmic}
\end{algorithm}

The algorithm does not require a diameter or a distance-to-solution parameter. The following assumption is used only to control the trajectory radius in the analysis.

\begin{assumption}[Bounded DA-Muon trajectory]
\label{assum:da_bounded_trajectory}
For the time horizon \(T\), there exists \(D<\infty\) such that the iterates generated by Algorithm~\ref{algo_dis_adaptive} satisfy
\[
    \|x_k-x_0\|\le D,
    \qquad
    k=0,\ldots,T.
\]
\end{assumption}

\begin{theorem}[Stationarity of Distance-Adaptive Muon]
\label{thm_distance_adaptive}
Let Assumptions~\ref{assum:smoothness} and \ref{assum:da_bounded_trajectory} hold, assume that \(f^\star:=\inf_{x\in\mathcal E} f(x)>-\infty\), and choose the analysis-only trajectory bound \(D\) so that \(D\ge r\). Then, for all \(T\ge1\), the iterates of Algorithm~\ref{algo_dis_adaptive} satisfy
\[
    \min_{0\le k\le T-1}
    \|\nabla f(x_{k+1})\|_*
    \le
    \frac{\sqrt{2}\bigl(f(x_0)-f^\star\bigr)}{r\sqrt T}
    +
    \frac{2L C_\alpha(T)D}{\sqrt T}
    \left(\frac Dr\right)^{2/T}
    \log\!\left(\frac{eD}{r}\right),
\]
where \(C_\alpha(T)=O_\alpha(1+\log T)\). The explicit expression for \(C_\alpha(T)\) and the proof are given in Appendix~\ref{app:da_muon}.
\end{theorem}

Theorem~\ref{thm_distance_adaptive} gives an \(O(\log T/\sqrt T)\) stationarity guarantee for smooth non-convex objectives. The trajectory bound \(D\) appears only in the analysis: the update rule uses only the observed radius \(\bar r_k\). The dependence on the initial radius \(r=\bar r_0\) is standard in distance-adaptive methods and reflects the cost of starting from a user-specified lower scale.

The non-convex result above gives stationarity but not an objective-gap rate. We next turn to star-convex objectives, where the global geometry is strong enough to prove objective-gap guarantees while avoiding bounded-domain and clipping-domain assumptions.
\section{Scale-Calibrated Muon for Star-Convex Functions}
\label{sec:scale_calibrated_muon}

SC-Muon is a star-convex scale rule that is distance-free at the algorithmic level: it does not use the distance to a minimizer, the diameter of a prescribed clipping domain, a trajectory radius, or the level-set radius \(D_{\rm lev}\). The radius \(D_{\rm lev}\) appears only in the analysis through the bounded initial sublevel-set assumption below.

\begin{assumption}[Bounded initial sublevel set]
\label{assum:bounded_level}
Let \(x^\star\) be the global minimizer from Assumption~\ref{assum:star_convex}. There exists \(D_{\rm lev}<\infty\) such that
\[
    f(x)\le f(x_0)
    \quad\Longrightarrow\quad
    \|x-x^\star\|\le D_{\rm lev}.
\]
\end{assumption}

Assumption~\ref{assum:bounded_level} is a standard initial-level-set localization condition in descent, trust-region, and regularized Newton-type analyses~\citep{conn2000trust,nesterov2008accelerating,doikov2024super}. The descent inequality below implies that Algorithm~\ref{alg:scale_calibrated_muon} is monotone. The formal proof is given in Appendix~\ref{app:sc_muon}. Hence its iterates remain in this set. The assumption is therefore local to the trajectory and weaker than assuming a bounded clipping domain.

Let \(g_k:=\nabla f(x_k)\). The method maintains the same exponential moving average as Muon,
\[
    m_{k+1}=(1-\alpha)m_k+\alpha g_k,
    \qquad
    \alpha\in(0,1),
\]
initialized by \(m_0=g_0\). The Muon direction is chosen by the linear minimization oracle
\[
    u_k\in\argmax_{\|u\|\le1}\langle m_{k+1},u\rangle,
    \qquad
    \langle m_{k+1},u_k\rangle=\|m_{k+1}\|_*.
\]
Because \(m_{k+1}\) may be stale, we certify its alignment with the current gradient:
\[
    \langle g_k,u_k\rangle
    =
    \langle m_{k+1},u_k\rangle
    +
    \langle g_k-m_{k+1},u_k\rangle
    \ge
    \|m_{k+1}\|_*-\|g_k-m_{k+1}\|_*.
\]
We define
\[
    e_k:=\|g_k-m_{k+1}\|_*,
    \qquad
    a_k:=\left(\|m_{k+1}\|_*-e_k\right)_+,
\]
and set
\[
    \eta_k:=\frac{a_k}{L},
    \qquad
    x_{k+1}=x_k-\eta_k u_k.
\]
If \(a_k>0\), then the update direction \(-u_k\) is certified as a descent direction for the current gradient. If the momentum direction is not certified, the positive part reduces the step to zero. By \(L\)-smoothness,
\[
    f(x_{k+1})
    \le
    f(x_k)
    -
    \frac{a_k^2}{2L},
\]
so the method is monotone and all iterates remain in the initial sublevel set.

\begin{algorithm}[t]
\caption{Scale-Calibrated Muon}
\label{alg:scale_calibrated_muon}
\begin{algorithmic}[1]
\REQUIRE Initial point \(x_0\in\mathcal E\), smoothness constant \(L\), momentum parameter \(\alpha\in(0,1)\)
\STATE \(g_0\gets\nabla f(x_0)\), \(m_0\gets g_0\)
\FOR{\(k=0,1,\dots,T-1\)}
    \STATE \(g_k\gets\nabla f(x_k)\)
    \STATE \(m_{k+1}\gets(1-\alpha)m_k+\alpha g_k\)
    \STATE \(e_k\gets\|g_k-m_{k+1}\|_*\)
    \STATE Choose \(u_k\in\argmax_{\|u\|\le1}\langle m_{k+1},u\rangle\)
    \STATE \(a_k\gets\left(\|m_{k+1}\|_*-e_k\right)_+\)
    \STATE \(\eta_k\gets a_k/L\)
    \STATE \(x_{k+1}\gets x_k-\eta_k u_k\)
\ENDFOR
\end{algorithmic}
\end{algorithm}

\begin{theorem}[Last-iterate rate for Scale-Calibrated Muon]
\label{thm:certified_rate}
Let Assumptions~\ref{assum:smoothness}, \ref{assum:star_convex}, and \ref{assum:bounded_level} hold, and consider Algorithm~\ref{alg:scale_calibrated_muon}. Then, for all \(T\ge1\),
\[
    f(x_T)-f(x^\star)
    \le
    K_\alpha\frac{LD_{\rm lev}^2}{T},
\]
where \(K_\alpha>0\) depends only on the momentum parameter \(\alpha\).
\end{theorem}

The proof is given in Appendix~\ref{app:sc_muon}. It uses a Lyapunov function combining the objective gap with the squared momentum-tracking error, which allows the stale momentum error to be amortized over time.

DF-Muon, introduced next, uses a majorized scalar search over the trust-region radius and does not require the bounded initial sublevel-set assumption.
\section{Distance-Free Muon for Star-Convex Functions}
\label{sec:df_muon}

DF-Muon is the star-convex variant that removes the distance-to-solution input without assuming bounded iterates, bounded gradients, a bounded domain, or a bounded initial sublevel set. It chooses the trust-region radius by a regularized one-dimensional search on a smoothness majorant, rather than by evaluating the true objective along the ray.

Muon-type trust-region analyses typically require a radius proportional to the unknown distance from the initialization to a global minimizer~\citep{kovalev2025understanding}. Let
\[
    D:=\|x_0-x^\star\|.
\]
Distance-Free Muon does not take \(D\) as an input. Instead, it uses three components: a one-dimensional lower certificate for \(D\), the Muon linear-minimization direction, and a majorized scalar search that competes with the hidden feasible radius \(R=D\).

\paragraph{Distance certificate.}
The method maintains a scalar proxy \(d_k\), inspired by the D-adaptation principle of constructing lower certificates for the unknown distance-to-solution \citep{defazio2023learning}. For arbitrary weights \(\omega_k\ge0\), let
\begin{equation}
\label{eq:df_proxy_update}
\begin{gathered}
    S_{k+1}=S_k+\omega_k g_k,\qquad
    B_{k+1}=B_k-\omega_k\langle g_k,x_k-x_0\rangle,\\
    \widehat d_{k+1}
    =
    \begin{cases}
    [B_{k+1}]_+/\|S_{k+1}\|_*,
    & \text{if }\|S_{k+1}\|_*>0,\\
    0,
    & \text{if }\|S_{k+1}\|_*=0,
    \end{cases}
    \qquad
    d_{k+1}=\max\{d_k,\widehat d_{k+1}\}.
\end{gathered}
\end{equation}
As shown in Appendix~\ref{app:df_muon}, if \(0\le d_0\le D\), then \(d_k\le D\) for all \(k\). Unlike standard D-adaptation, our proof does not require \(d_k\) to recover \(D\). The proxy only regularizes the scalar radius search below.

\paragraph{Muon direction and recentered ray.}
Given exponential momentum
\[
    m_{k+1}=(1-\alpha)m_k+\alpha g_k,
\]
Distance-Free Muon keeps the Muon direction
\begin{equation}
\label{eq:df_muon_direction}
    s_k\in\operatorname*{arg\,min}_{\|s\|\le1}
    \langle m_{k+1},s\rangle .
\end{equation}
In the spectral--nuclear matrix geometry, this is exactly the orthogonalized momentum direction used by Muon-type methods. The trust-region path is recentered toward the initialization:
\begin{equation}
\label{eq:df_recentered_ray}
    c_k:=x_0+(1-\beta)(x_k-x_0),
    \qquad
    z_k(R):=c_k+\beta R s_k .
\end{equation}
Equivalently, with \(y_k:=x_k-x_0\),
\[
    z_k(R)-x_k=\beta(Rs_k-y_k),
    \qquad
    z_k(R)-x_0=(1-\beta)y_k+\beta R s_k .
\]

\paragraph{Majorized scalar search.}
Instead of setting \(R=D\), which is unknown, the method chooses \(R_k\) by minimizing a regularized smoothness majorant along the recentered ray. Define
\begin{equation}
\label{eq:df_majorant}
\begin{aligned}
    \mathcal Q_k(R)
    :=
    f(x_k)
    &+
    \beta\langle g_k,R s_k-y_k\rangle
    +
    \frac{L\beta^2}{2}\|R s_k-y_k\|^2 \\
    &+
    M\beta L\|(1-\beta)y_k+\beta R s_k\|^2
    +
    \rho L\beta^2\|R s_k-y_k\|^2 \\
    &+
    \frac{\lambda L\beta^2}{2}(R-d_{k+1})^2 .
\end{aligned}
\end{equation}
Then
\begin{equation}
\label{eq:df_radius_search}
    R_k\in\operatorname*{arg\,min}_{R\ge0}\mathcal Q_k(R),
    \qquad
    x_{k+1}=z_k(R_k).
\end{equation}
The scalar objective \(\mathcal Q_k\) is convex in \(R\). It is a linear term plus squared norms of affine functions of \(R\), together with a quadratic regularizer. Although \(D\) is unknown to the algorithm, the search is over all \(R\ge0\). The proof can therefore compare the chosen radius to the hidden feasible radius \(R=D\), recovering the scale used in fixed-radius trust-region analyses without passing \(D\) to the method.

\begin{algorithm}[H]
\caption{Distance-Free Muon}
\label{alg:pf_tr_muon}
\begin{algorithmic}[1]
\REQUIRE Initial point \(x_0\in\mathcal E\), horizon \(T\), initial proxy \(d_0=0\), \(\alpha,\beta\in(0,1)\), \(\rho,\lambda,M>0\)
\STATE \(m_0\gets\nabla f(x_0)\), \(S_0\gets0\), \(B_0\gets0\)
\FOR{\(k=0,1,\ldots,T-1\)}
    \STATE \(g_k\gets\nabla f(x_k)\)
    \STATE Choose \(\omega_k\ge0\) and update \(d_{k+1}\) by \eqref{eq:df_proxy_update}
    \STATE Update \(m_{k+1}\) and choose \(s_k\) by \eqref{eq:df_muon_direction}
    \STATE Define \(c_k,z_k(R)\) by \eqref{eq:df_recentered_ray}
    \STATE Choose \(R_k\) by \eqref{eq:df_radius_search}
    \STATE \(x_{k+1}\gets z_k(R_k)\)
\ENDFOR
\STATE \textbf{return} \(x_T\)
\end{algorithmic}
\end{algorithm}

\begin{theorem}[Last-iterate convergence of Distance-Free Muon]
\label{thm:df_muon}
Let Assumptions~\ref{assum:smoothness} and~\ref{assum:star_convex} hold. Let Algorithm~\ref{alg:pf_tr_muon} be run with \(d_0=0\), or more generally with any \(0\le d_0\le D\), where \(D=\|x_0-x^\star\|\). Suppose \(M\ge 2(1+2\rho)\), \(\alpha,\beta\in(0,1)\), and \(\alpha>\beta/2\). Then, for every \(T\ge1\),
\[
\begin{aligned}
    f(x_T)-f(x^\star)
    \le
    \left(1-\frac{\beta}{2}\right)^T
    \bigl(f(x_0)-f(x^\star)\bigr)
    +
    C_{\alpha,\beta,\rho,\lambda,M}L\beta D^2,
\end{aligned}
\]
where
\[
    C_{\alpha,\beta,\rho,\lambda,M}
    :=
    2C_0+
    \frac{4(1-\alpha)^2}{\rho(\alpha-\beta/2)^2},
    \qquad
    C_0:=1+2\rho+\frac{\lambda}{2}+M .
\]
Consequently, choosing
\[
    \beta=\min\left\{\alpha,\frac{2\log(T+1)}{T}\right\}
\]
gives, whenever \(2\log(T+1)/T\le \alpha\),
\[
    f(x_T)-f(x^\star)
    =
    \widetilde O\!\left(
        \frac{L\|x_0-x^\star\|^2}{T}
    \right),
\]
where the hidden constants depend only on \(\alpha,\rho,\lambda,M\), but not on \(D\).
\end{theorem}

The theorem, proved in Appendix~\ref{app:df_muon}, shows that the unknown distance \(D=\|x_0-x^\star\|\) appears only in the analysis, not as an input to the algorithm. The result requires neither a bounded domain, bounded iterates, bounded gradients, nor a bounded initial sublevel set. The parameters \(\rho,\lambda,M\) are fixed regularization weights for the scalar search. They are not estimates of a problem-dependent scale. For example, one may take \(\rho=1\), \(\lambda=1\), and \(M=6\). Relative to \citet{kovalev2025understanding}, the method removes the distance-to-solution or boundedness parameter from the algorithm. Relative to an objective-value line search along the ray, the majorized search gives a convex one-dimensional subproblem.

\FloatBarrier
\section{Experimental Validation}
\label{sec:experiments}

AdamW is included as a standard reference optimizer. Since the experiments evaluate scalar adaptation for Muon-type normalized updates, the primary baseline is fixed Muon tuned by a learning-rate sweep. The adaptive Muon variants are therefore compared mainly against this tuned fixed-Muon baseline, rather than against a separately optimized AdamW baseline. DF-Muon is not tuned by an exhaustive task-specific search. We use small one-seed diagnostics to select simple stability or model coefficients, and otherwise keep the same practical majorized scalar rule across the experiments.

\subsection{GPT-124M on WikiText-103}
\label{subsec:gpt124m_wikitext103}

We evaluate the methods on GPT-124M/WikiText-103. We train a GPT-style decoder-only Transformer with approximately $124$M parameters for $5000$ optimizer steps on a single A100 GPU. We compare AdamW, a tuned fixed-Muon baseline, and the three adaptive Muon variants. The fixed-Muon learning rate is selected by a learning-rate sweep, which gives $\eta=0.005$ as the best fixed scale. For DA-Muon, we use the best cap from a separate cap sweep, $\eta_{\max}=0.01$. DF-Muon uses the practical Frobenius-proxy majorized scalar implementation described in Appendix~\ref{app:df_muon_implementation}. For GPT-124M/WikiText-103, we use the no-center variant selected by the one-seed majorized-model diagnostic in Appendix~\ref{app:gpt124m_wikitext103_diagnostics}; this removes the center penalty while keeping the same cap, floor, smoothing, grid, and refinement settings. SC-Muon uses the default adaptive settings and the same optimizer split and training schedule as the other Muon variants. Additional fixed-Muon learning-rate, DA-Muon cap, and observed-scale diagnostics are reported in Appendix~\ref{app:gpt124m_wikitext103_diagnostics}.

\begin{table}[t]
\centering
\caption{
GPT-style Transformer on WikiText-103 with approximately $124$M parameters. We report mean $\pm$ standard deviation over seeds $\{42,1337,2024\}$. The fixed Muon baseline uses the best learning rate from the sweep, $\eta=0.005$. DA-Muon uses the best cap from the cap sweep, $\eta_{\max}=0.01$. Runtime is reported relative to fixed Muon.
}
\label{tab:gpt124m_wikitext103_three_seeds}
\footnotesize
\setlength{\tabcolsep}{4pt}
\begin{tabular}{@{}lcccc@{}}
\toprule
\textbf{Method}
& \textbf{Train loss}
& \textbf{Val. loss}
& \textbf{Mean $\eta$}
& \textbf{Rel. time} \\
\midrule
AdamW
& $3.9964 \pm 0.0026$
& $3.9353 \pm 0.0143$
& $0.0010$
& $0.65\times$ \\
Best fixed Muon
& $3.7662 \pm 0.0029$
& $3.7166 \pm 0.0199$
& $0.0050$
& $1.00\times$ \\
DF-Muon
& $\mathbf{3.7192 \pm 0.0028}$
& $\mathbf{3.6714 \pm 0.0237}$
& $0.0300$
& $1.05\times$ \\
DA-Muon
& $3.7794 \pm 0.0017$
& $3.7273 \pm 0.0216$
& $0.0100$
& $1.03\times$ \\
SC-Muon
& $3.8589 \pm 0.0048$
& $3.8029 \pm 0.0180$
& $0.0263$
& $1.02\times$ \\
\bottomrule
\end{tabular}
\end{table}

\begin{figure}[t]
    \centering
    \includegraphics[width=0.92\linewidth]{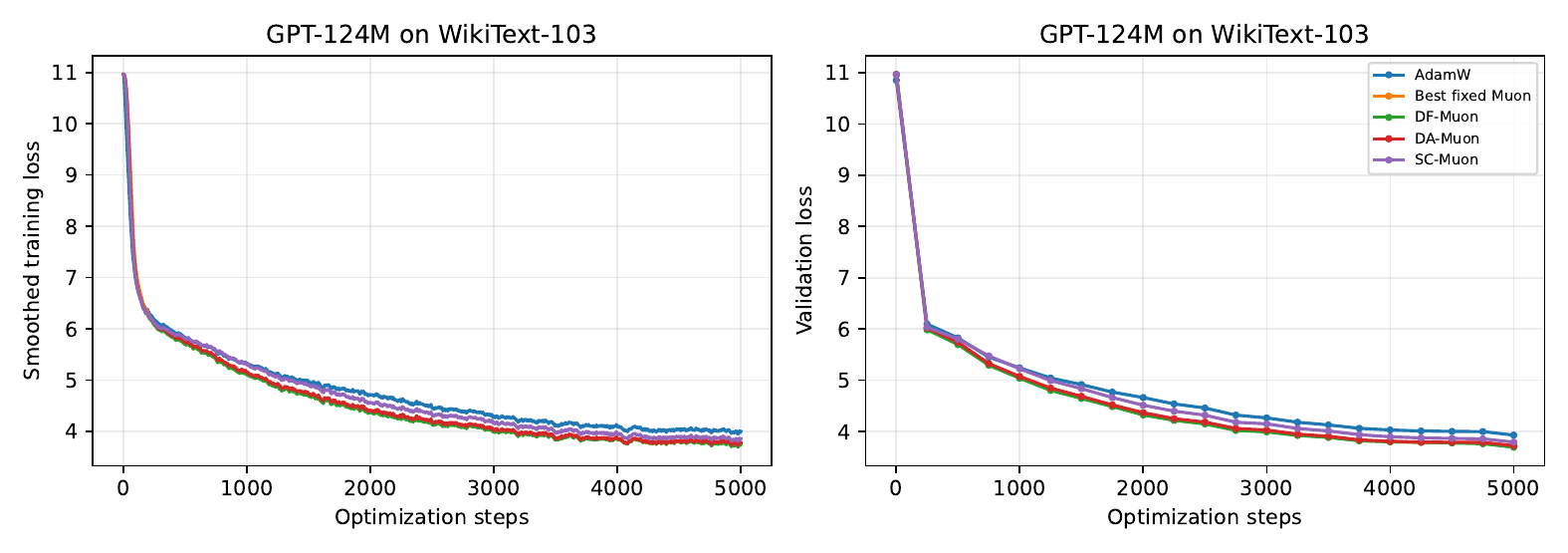}
    \caption{
    GPT-style Transformer training on WikiText-103 in a representative seed. We compare AdamW, the tuned fixed-Muon baseline, and the adaptive Muon variants. The curves are shown for representative training dynamics; the aggregate numbers in Table~\ref{tab:gpt124m_wikitext103_three_seeds} use the selected no-center DF-Muon variant.
    }
    \label{fig:gpt124m_wikitext103_train_val}
\end{figure}

DF-Muon gives the best mean training and validation loss on GPT-124M/WikiText-103. Against the tuned fixed-Muon baseline selected by the learning-rate sweep, DF-Muon improves mean validation loss from $3.7166$ to $3.6714$ and adds about $5\%$ runtime overhead. The selected no-center variant has mean base scale $0.0300$ over the final $20\%$ of training, matching its cap in this experiment. DA-Muon becomes competitive after cap tuning, but remains slightly worse than tuned fixed Muon on this benchmark. SC-Muon is stable but more conservative. These results support the practical role of the majorized DF-Muon scalar rule on this language-modeling task.

\subsection{Regularized ViT-Tiny on CIFAR-100}
\label{subsec:vit_cifar100_regularized}

We also evaluate the adaptive Muon scaling rules on image classification. We train a ViT-Tiny model on CIFAR-100, resizing images to \(224\times224\), using a regularized recipe with random resized crops, RandAugment, random erasing, label smoothing, and a common warmup--cosine schedule. All methods are run for \(100\) epochs over seeds \(\{42,1337,2024\}\). The fixed-Muon baseline uses \(\eta=0.001\), selected from a fixed-Muon learning-rate diagnostic. DF-Muon uses \(\eta_{\max}=0.01\), selected from a one-seed cap diagnostic under the same regularized ViT/CIFAR-100 setup. As in the language-modeling experiments, the Muon update is applied to matrix-valued parameters, while remaining parameters are handled by AdamW. The fixed-Muon learning-rate diagnostic, the DF-Muon cap diagnostic, and the observed effective-scale plots are reported in Appendix~\ref{app:vit_cifar100_regularized_diagnostics}.

DF-Muon uses the practical implementation described in Appendix~\ref{app:df_muon_implementation}. Its cap is selected by the diagnostic in Appendix~\ref{app:vit_cifar100_regularized_diagnostics}.

\begin{table}[t]
\centering
\caption{
Final regularized ViT-Tiny/CIFAR-100 experiment. CIFAR-100 images are resized to \(224\times224\). We report best validation cross-entropy and best top-1 accuracy over epochs. Runtime is relative to fixed Muon. The fixed Muon baseline uses the tuned learning rate \(\eta=0.001\), and DF-Muon uses \(\eta_{\max}=0.01\) selected from the cap diagnostic.
}
\label{tab:vit_cifar100_regularized}
\footnotesize
\setlength{\tabcolsep}{3.5pt}
\begin{tabular}{@{}lccccc@{}}
\toprule
\textbf{Method}
& \textbf{Best Val. CE}
& \textbf{Top-1 @ Best CE}
& \textbf{Best Top-1}
& \textbf{Mean \(\eta\)}
& \textbf{Rel. time} \\
\midrule
AdamW
& \(1.5561 \pm 0.0196\)
& \(60.04 \pm 0.85\)
& \(63.23 \pm 0.45\)
& \(0.0003\)
& \(0.83\times\) \\
Best fixed Muon
& \(\mathbf{1.3851 \pm 0.0131}\)
& \(64.15 \pm 1.28\)
& \(67.69 \pm 0.58\)
& \(0.0010\)
& \(1.00\times\) \\
DF-Muon
& \(1.3974 \pm 0.0132\)
& \(63.32 \pm 0.67\)
& \(\mathbf{67.93 \pm 0.26}\)
& \(0.0100\)
& \(1.05\times\) \\
DA-Muon
& \(1.5124 \pm 0.0209\)
& \(61.46 \pm 0.55\)
& \(63.38 \pm 0.52\)
& \(0.0100\)
& \(1.02\times\) \\
SC-Muon
& \(1.4392 \pm 0.0162\)
& \(63.19 \pm 0.51\)
& \(65.28 \pm 0.17\)
& \(0.0168\)
& \(1.02\times\) \\
\bottomrule
\end{tabular}
\end{table}

\begin{figure}[t]
    \centering
    \includegraphics[width=0.92\linewidth]{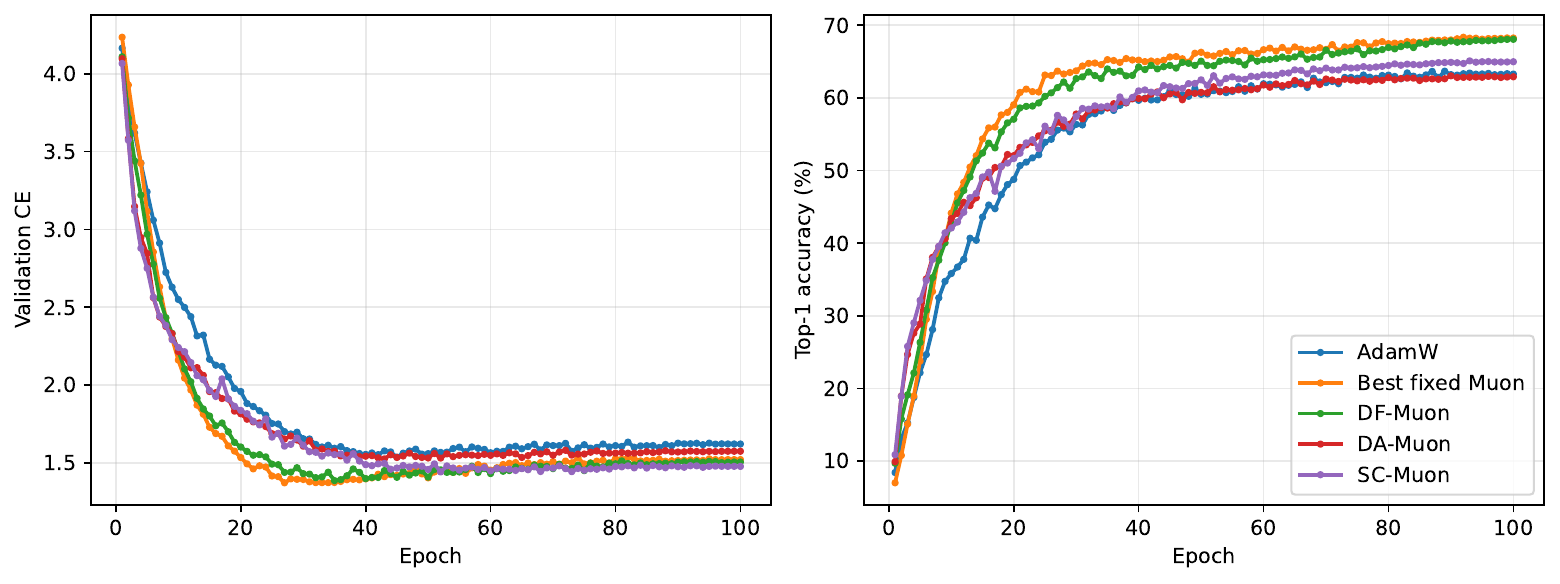}
    \caption{
    Regularized ViT-Tiny/CIFAR-100 training dynamics in a representative seed. The left panel shows validation cross-entropy, and the right panel shows top-1 accuracy. Table~\ref{tab:vit_cifar100_regularized} reports the aggregate over three seeds.
    }
    \label{fig:vit_cifar100_regularized_curves}
\end{figure}

On regularized ViT-Tiny/CIFAR-100, DF-Muon remains competitive with the tuned fixed-Muon baseline. Fixed Muon obtains the lowest best validation cross-entropy, while DF-Muon obtains the highest mean best top-1 accuracy and has only a \(5\%\) runtime overhead. DA-Muon and SC-Muon are stable under the regularized recipe but do not match the top-1 accuracy of DF-Muon or tuned fixed Muon. These results indicate that the DF-Muon scalar rule transfers beyond language modeling and can match a tuned fixed-Muon scale on a vision Transformer benchmark.

A complementary CIFAR-100/ResNet-32 diagnostic is reported in Appendix~\ref{app:cifar100_resnet32_diagnostics}.
\section{Related Work}
\label{sec:related_work}

\paragraph{Muon and LMO-based normalized optimization.}
Muon and related normalized optimizers are motivated by the empirical success of gradient orthogonalization in neural-network training~\citep{jordan2024muon,large2024scalable,liu2025muon,shah2025practical,wen2025fantastic}. The spectral-norm steepest-descent interpretation of orthogonalized updates was developed by \citet{bernstein2024old}, with earlier connections to stochastic and preconditioned spectral descent~\citep{carlson2015stochastic,carlson2015preconditioned,kovalev2018stochastic}. Early convergence analyses connected Muon to normalized SGD with momentum and spectral-norm constrained optimization~\citep{cutkosky2020momentum,li2025note,shen2025convergence,chen2025muon}. A key theoretical development is the non-Euclidean trust-region interpretation of \citet{kovalev2025understanding}, which shows that Muon-type updates can be viewed as linear-minimization-oracle steps in the spectral--nuclear matrix geometry. This connects Muon to the classical Frank--Wolfe/LMO framework~\citep{frank1956algorithm,jaggi2013revisiting} and to recent LMO-based optimizers such as Scion, Gluon, and related momentum variants~\citep{pethick2025scion,riabinin2025gluon,khirirat2025lmo}. Related work on generalized norm clipping and non-Euclidean smoothness further clarifies the role of normalized updates beyond Euclidean Lipschitz smoothness~\citep{pethick2025generalized}. These works primarily clarify or improve the geometry of the update direction. Our work is complementary: we keep the Muon/LMO direction and study how to choose the scalar trust-region radius multiplying that direction.

\paragraph{Parameter-free online learning and comparator adaptivity.}
A classical route to removing scale parameters comes from online learning. Unconstrained online linear learning and coin-betting methods obtain regret bounds that adapt to the unknown comparator norm~\citep{mcmahan2014unconstrained,orabona2016coin}. Subsequent work developed algorithms that require little prior information, handle unbounded online domains through artificial constraints or hints, and extend parameter-free guarantees to Banach spaces~\citep{cutkosky2017online,cutkosky2019artificial,cutkosky2018black}. Other refinements include scale-invariant online algorithms and comparator-norm/Lipschitz adaptivity~\citep{kempka2019adaptive,mhammedi2020lipschitz}. These methods can be converted to stochastic optimization by online-to-batch arguments, and coin-betting ideas have also led to practical learning-rate-free neural-network training methods~\citep{orabona2017training}. In contrast, \citet{carmon2022making} move beyond the online-to-batch abstraction by deriving an implicit step-size certificate for SGD and solving it by logarithmic bisection. Our work follows the same high-level goal of removing hidden distance and scale parameters, but in a different geometry: the update direction is produced by a Muon/LMO oracle, and the object to be calibrated is the scalar trust-region radius.

\paragraph{Distance-adaptive stochastic optimization.}
More recent parameter-free optimization methods estimate distance or scale quantities directly from the trajectory. DoG chooses the step size from the maximum distance traveled from initialization and accumulated gradient norms~\citep{ivgi2023dog}. DoWG refines this distance-over-gradients principle using a weighted normalization and obtains universal parameter-free guarantees for convex optimization~\citep{khaled2023dowg}. D-Adaptation estimates the unknown distance-to-solution scale online and provides learning-rate-free variants of standard optimizers~\citep{defazio2023learning}. DADA extends distance adaptation to dual averaging and obtains guarantees for broad classes of convex problems, including unbounded domains~\citep{moshtaghifar2025dada}. Our methods are inspired by this distance-adaptive philosophy, but apply it to Muon-type normalized trust-region updates rather than to raw-gradient or dual-averaging updates.

\paragraph{Star-convexity and Muon-type guarantees.}
Star-convexity is a structured non-convex condition that preserves a first-order inequality toward a global minimizer. It has been used as an analytic model for favorable non-convex geometry in neural-network optimization, including one-point convexity and star-convex path conditions~\citep{kleinberg2018alternative,zhou2019sgd}. The closest theoretical work to our star-convex results is \citet{kovalev2025understanding}, who proves non-convex and star-convex guarantees for non-Euclidean trust-region methods that recover Muon-type updates in the spectral--nuclear geometry. In the star-convex case, the resulting radius choices depend on distance or localization quantities, such as a radius proportional to the distance from the initialization to a minimizer or assumptions that localize the feasible region. Our Distance-Free Muon result keeps the Muon/LMO direction but chooses the scalar radius by a majorized scalar search, obtaining a last-iterate star-convex guarantee without assuming bounded iterates, bounded gradients, a bounded domain, or a bounded initial sublevel set.

\paragraph{Regularized model methods.}
The majorized scalar search in DF-Muon is also related in spirit to regularized model methods such as cubic regularization and adaptive regularization, although our subproblem is one-dimensional, first-order, and tied to the Muon/LMO direction~\citep{nesterov2006cubic,cartis2011adaptive1,cartis2011adaptive2}.
\section{Conclusion}
\label{sec:conclusion}

We studied adaptive scalar scaling for Muon-type normalized optimizers. The analysis separates the geometry-aware Muon direction from the scalar radius multiplying that direction. We gave three complementary scale rules: a trajectory-radius rule for smooth non-convex stationarity, a certificate-based rule for star-convex last-iterate convergence under a bounded initial sublevel set, and a distance-free star-convex rule that removes the distance-to-solution input from the algorithm. Experiments on language modeling and image classification show that these scale-selection rules remain useful in stochastic neural-network training: DF-Muon improves over tuned fixed-scale Muon on GPT-124M/WikiText-103, while the adaptive variants remain competitive across the vision diagnostics. Stochastic convergence theory and approximate majorized-search variants are natural directions for future work.

\bibliographystyle{apalike}
\bibliography{reference}

@article{kovalev2025understanding,
  title={Understanding gradient orthogonalization for deep learning via non-euclidean trust-region optimization},
  author={Kovalev, Dmitry},
  journal={arXiv preprint arXiv:2503.12645},
  year={2025}
}

@article{khaled2023dowg,
  title={DoWG unleashed: An efficient universal parameter-free gradient descent method},
  author={Khaled, Ahmed and Mishchenko, Konstantin and Jin, Chi},
  journal={Advances in Neural Information Processing Systems},
  volume={36},
  pages={6748--6769},
  year={2023}
}

@inproceedings{defazio2023learning,
  title={Learning-rate-free learning by d-adaptation},
  author={Defazio, Aaron and Mishchenko, Konstantin},
  booktitle={International Conference on Machine Learning},
  pages={7449--7479},
  year={2023},
  organization={PMLR}
}

@inproceedings{ivgi2023dog,
  title={DoG is SGD’s best friend: A parameter-free dynamic step size schedule},
  author={Ivgi, Maor and Hinder, Oliver and Carmon, Yair},
  booktitle={International Conference on Machine Learning},
  pages={14465--14499},
  year={2023},
  organization={PMLR}
}

@article{liu2023stochastic,
  title={Stochastic nonsmooth convex optimization with heavy-tailed noises: High-probability bound, in-expectation rate and initial distance adaptation},
  author={Liu, Zijian and Zhou, Zhengyuan},
  journal={arXiv preprint arXiv:2303.12277},
  year={2023}
}

@article{moshtaghifar2025dada,
  title={DADA: Dual Averaging with Distance Adaptation},
  author={Moshtaghifar, Mohammad and Rodomanov, Anton and Vankov, Daniil and Stich, Sebastian},
  journal={arXiv preprint arXiv:2501.10258},
  year={2025}
}

@article{chakraborty2026randomized,
  title={Randomized Feasibility Methods for Constrained Optimization with Adaptive Step Sizes},
  author={Chakraborty, Abhishek and Nedi{\'c}, Angelia},
  journal={arXiv preprint arXiv:2601.20076},
  year={2026}
}

@article{duchi2011adaptive,
  title={Adaptive subgradient methods for online learning and stochastic optimization.},
  author={Duchi, John and Hazan, Elad and Singer, Yoram},
  journal={Journal of machine learning research},
  volume={12},
  number={7},
  year={2011}
}

@article{kingma2015adam,
  title={Adam: A method for stochastic optimization},
  author={Kingma, Diederik P and Ba, Jimmy},
  journal={arXiv preprint arXiv:1412.6980},
  year={2014}
}

@article{loshchilov2019decoupled,
  title={Decoupled weight decay regularization},
  author={Loshchilov, Ilya and Hutter, Frank},
  journal={arXiv preprint arXiv:1711.05101},
  year={2017}
}

@article{jordan2024muon,
  title={Muon: An optimizer for hidden layers in neural networks, 2024},
  author={Jordan, Keller and Jin, Yuchen and Boza, Vlado and Jiacheng, You and Cesista, Franz and Newhouse, Laker and Bernstein, Jeremy},
  journal={URL https://kellerjordan.github.io/posts/muon/},
  volume={6},
  number={3},
  pages={4},
  year={2024}
}

@article{large2024scalable,
  title={Scalable optimization in the modular norm},
  author={Large, Tim and Liu, Yang and Huh, Minyoung and Bahng, Hyojin and Isola, Phillip and Bernstein, Jeremy},
  journal={Advances in Neural Information Processing Systems},
  volume={37},
  pages={73501--73548},
  year={2024}
}

@article{pethick2025scion,
  title={Training deep learning models with norm-constrained lmos},
  author={Pethick, Thomas and Xie, Wanyun and Antonakopoulos, Kimon and Zhu, Zhenyu and Silveti-Falls, Antonio and Cevher, Volkan},
  journal={arXiv preprint arXiv:2502.07529},
  year={2025}
}

@article{riabinin2025gluon,
  title={Gluon: Making muon \& scion great again!(bridging theory and practice of lmo-based optimizers for llms)},
  author={Riabinin, Artem and Shulgin, Egor and Gruntkowska, Kaja and Richt{\'a}rik, Peter},
  journal={arXiv preprint arXiv:2505.13416},
  year={2025}
}

@article{khirirat2025lmo,
  title={Better LMO-based Momentum Methods with Second-Order Information},
  author={Khirirat, Sarit and Sadiev, Abdurakhmon and Demidovich, Yury and Richt{\'a}rik, Peter},
  journal={arXiv preprint arXiv:2512.13227},
  year={2025}
}

@inproceedings{cutkosky2020momentum,
  title={Momentum improves normalized sgd},
  author={Cutkosky, Ashok and Mehta, Harsh},
  booktitle={International conference on machine learning},
  pages={2260--2268},
  year={2020},
  organization={PMLR}
}

@inproceedings{carmon2022making,
  title={Making SGD parameter-free},
  author={Carmon, Yair and Hinder, Oliver},
  booktitle={Conference on learning theory},
  pages={2360--2389},
  year={2022},
  organization={PMLR}
}

@article{liu2025muon,
  title={Muon is scalable for llm training},
  author={Liu, Jingyuan and Su, Jianlin and Yao, Xingcheng and Jiang, Zhejun and Lai, Guokun and Du, Yulun and Qin, Yidao and Xu, Weixin and Lu, Enzhe and Yan, Junjie and others},
  journal={arXiv preprint arXiv:2502.16982},
  year={2025}
}

@article{orabona2016coin,
  title={Coin betting and parameter-free online learning},
  author={Orabona, Francesco and P{\'a}l, D{\'a}vid},
  journal={Advances in Neural Information Processing Systems},
  volume={29},
  year={2016}
}

@inproceedings{cutkosky2018black,
  title={Black-box reductions for parameter-free online learning in banach spaces},
  author={Cutkosky, Ashok and Orabona, Francesco},
  booktitle={Conference On Learning Theory},
  pages={1493--1529},
  year={2018},
  organization={PMLR}
}

@article{bengio2000gradient,
  title={Gradient-based optimization of hyperparameters},
  author={Bengio, Yoshua},
  journal={Neural computation},
  volume={12},
  number={8},
  pages={1889--1900},
  year={2000},
  publisher={MIT Press One Rogers Street, Cambridge, MA 02142-1209, USA journals-info~…}
}

@incollection{feurer2019hyperparameter,
  title={Hyperparameter optimization},
  author={Feurer, Matthias and Hutter, Frank},
  booktitle={Automated machine learning: Methods, systems, challenges},
  pages={3--33},
  year={2019},
  publisher={Springer}
}

@inproceedings{kleinberg2018alternative,
  title={An alternative view: When does SGD escape local minima?},
  author={Kleinberg, Bobby and Li, Yuanzhi and Yuan, Yang},
  booktitle={International conference on machine learning},
  pages={2698--2707},
  year={2018},
  organization={PMLR}
}

@article{zhou2019sgd,
  title={Sgd converges to global minimum in deep learning via star-convex path},
  author={Zhou, Yi and Yang, Junjie and Zhang, Huishuai and Liang, Yingbin and Tarokh, Vahid},
  journal={arXiv preprint arXiv:1901.00451},
  year={2019}
}

@book{conn2000trust,
  title={Trust region methods},
  author={Conn, Andrew R and Gould, Nicholas IM and Toint, Philippe L},
  year={2000},
  publisher={SIAM}
}

@article{nesterov2008accelerating,
  title={Accelerating the cubic regularization of Newton’s method on convex problems},
  author={Nesterov, Yu},
  journal={Mathematical Programming},
  volume={112},
  number={1},
  pages={159--181},
  year={2008},
  publisher={Springer}
}

@article{doikov2024super,
  title={Super-universal regularized Newton method},
  author={Doikov, Nikita and Mishchenko, Konstantin and Nesterov, Yurii},
  journal={SIAM Journal on Optimization},
  volume={34},
  number={1},
  pages={27--56},
  year={2024},
  publisher={SIAM}
}

@inproceedings{mcmahan2014unconstrained,
  title={Unconstrained online linear learning in hilbert spaces: Minimax algorithms and normal approximations},
  author={McMahan, H Brendan and Orabona, Francesco},
  booktitle={Conference on Learning Theory},
  pages={1020--1039},
  year={2014},
  organization={PMLR}
}

@inproceedings{cutkosky2017online,
  title={Online learning without prior information},
  author={Cutkosky, Ashok and Boahen, Kwabena},
  booktitle={Conference on learning theory},
  pages={643--677},
  year={2017},
  organization={PMLR}
}

@inproceedings{cutkosky2019artificial,
  title={Artificial constraints and hints for unbounded online learning},
  author={Cutkosky, Ashok},
  booktitle={Conference on Learning Theory},
  pages={874--894},
  year={2019},
  organization={PMLR}
}

@article{orabona2017training,
  title={Training deep networks without learning rates through coin betting},
  author={Orabona, Francesco and Tommasi, Tatiana},
  journal={Advances in neural information processing systems},
  volume={30},
  year={2017}
}

@inproceedings{kempka2019adaptive,
  title={Adaptive scale-invariant online algorithms for learning linear models},
  author={Kempka, Michal and Kotlowski, Wojciech and Warmuth, Manfred K},
  booktitle={International conference on machine learning},
  pages={3321--3330},
  year={2019},
  organization={PMLR}
}

@inproceedings{mhammedi2020lipschitz,
  title={Lipschitz and comparator-norm adaptivity in online learning},
  author={Mhammedi, Zakaria and Koolen, Wouter M},
  booktitle={Conference on Learning Theory},
  pages={2858--2887},
  year={2020},
  organization={PMLR}
}

@article{shah2025practical,
  title={Practical efficiency of muon for pretraining},
  author={Shah, Ishaan and Polloreno, Anthony M and Stratos, Karl and Monk, Philip and Chaluvaraju, Adarsh and Hojel, Andrew and Ma, Andrew and Thomas, Anil and Tanwer, Ashish and Shah, Darsh J and others},
  journal={arXiv preprint arXiv:2505.02222},
  year={2025}
}

@article{wen2025fantastic,
  title={Fantastic pretraining optimizers and where to find them},
  author={Wen, Kaiyue and Hall, David and Ma, Tengyu and Liang, Percy},
  journal={arXiv preprint arXiv:2509.02046},
  year={2025}
}

@article{bernstein2024old,
  title={Old optimizer, new norm: An anthology},
  author={Bernstein, Jeremy and Newhouse, Laker},
  journal={arXiv preprint arXiv:2409.20325},
  year={2024}
}

@article{li2025note,
  title={A note on the convergence of muon and further},
  author={Li, Jiaxiang and Hong, Mingyi},
  journal={arXiv e-prints},
  pages={arXiv--2502},
  year={2025}
}

@article{shen2025convergence,
  title={On the convergence analysis of muon},
  author={Shen, Wei and Huang, Ruichuan and Huang, Minhui and Shen, Cong and Zhang, Jiawei},
  journal={arXiv preprint arXiv:2505.23737},
  year={2025}
}

@article{chen2025muon,
  title={Muon optimizes under spectral norm constraints},
  author={Chen, Lizhang and Li, Jonathan and Liu, Qiang},
  journal={arXiv preprint arXiv:2506.15054},
  year={2025}
}

@article{frank1956algorithm,
  title={An algorithm for quadratic programming},
  author={Frank, Marguerite and Wolfe, Philip},
  journal={Naval research logistics quarterly},
  volume={3},
  number={1-2},
  pages={95--110},
  year={1956},
  publisher={Wiley Online Library}
}

@inproceedings{jaggi2013revisiting,
  title={Revisiting Frank-Wolfe: Projection-free sparse convex optimization},
  author={Jaggi, Martin},
  booktitle={International conference on machine learning},
  pages={427--435},
  year={2013},
  organization={PMLR}
}

@article{nesterov2006cubic,
  title={Cubic regularization of Newton method and its global performance},
  author={Nesterov, Yurii and Polyak, Boris T},
  journal={Mathematical programming},
  volume={108},
  number={1},
  pages={177--205},
  year={2006},
  publisher={Springer}
}

@article{cartis2011adaptive1,
  title={Adaptive cubic regularisation methods for unconstrained optimization. Part I: motivation, convergence and numerical results},
  author={Cartis, Coralia and Gould, Nicholas IM and Toint, Philippe L},
  journal={Mathematical Programming},
  volume={127},
  number={2},
  pages={245--295},
  year={2011},
  publisher={Springer}
}

@article{cartis2011adaptive2,
  title={Adaptive cubic regularisation methods for unconstrained optimization. Part II: worst-case function-and derivative-evaluation complexity},
  author={Cartis, Coralia and Gould, Nicholas IM and Toint, Philippe L},
  journal={Mathematical programming},
  volume={130},
  number={2},
  pages={295--319},
  year={2011},
  publisher={Springer}
}

@inproceedings{carlson2015stochastic,
  title={Stochastic spectral descent for restricted Boltzmann machines},
  author={Carlson, David and Cevher, Volkan and Carin, Lawrence},
  booktitle={Artificial intelligence and statistics},
  pages={111--119},
  year={2015},
  organization={PMLR}
}

@article{carlson2015preconditioned,
  title={Preconditioned spectral descent for deep learning},
  author={Carlson, David E and Collins, Edo and Hsieh, Ya-Ping and Carin, Lawrence and Cevher, Volkan},
  journal={Advances in neural information processing systems},
  volume={28},
  year={2015}
}

@article{kovalev2018stochastic,
  title={Stochastic spectral and conjugate descent methods},
  author={Kovalev, Dmitry and Richtarik, Peter and Gorbunov, Eduard and Gasanov, Elnur},
  journal={Advances in Neural Information Processing Systems},
  volume={31},
  year={2018}
}

@article{pethick2025generalized,
  title={Generalized Gradient Norm Clipping \& Non-Euclidean $(L_0,L_1)$-Smoothness},
  author={Pethick, Thomas and Xie, Wanyun and Erdogan, Mete and Antonakopoulos, Kimon and Silveti-Falls, Tony and Cevher, Volkan},
  journal={arXiv preprint arXiv:2506.01913},
  year={2025}
}

@article{amsel2025polar,
  title={The polar express: Optimal matrix sign methods and their application to the muon algorithm},
  author={Amsel, Noah and Persson, David and Musco, Christopher and Gower, Robert M},
  journal={arXiv preprint arXiv:2505.16932},
  year={2025}
}

@article{qian2025muon,
  title={Muon is Provably Faster with Momentum Variance Reduction},
  author={Qian, Xun and Rammal, Hussein and Kovalev, Dmitry and Richtarik, Peter},
  journal={arXiv preprint arXiv:2512.16598},
  year={2025}
}

@article{gruntkowska2025error,
  title={Error feedback for Muon and friends},
  author={Gruntkowska, Kaja and Gaponov, Alexander and Tovmasyan, Zhirayr and Richt{\'a}rik, Peter},
  journal={arXiv preprint arXiv:2510.00643},
  year={2025}
}

@article{kovalev2025non,
  title={Non-Euclidean SGD for Structured Optimization: Unified Analysis and Improved Rates},
  author={Kovalev, Dmitry and Borodich, Ekaterina},
  journal={arXiv preprint arXiv:2511.11466},
  year={2025}
}

@article{merity2016pointer,
  title={Pointer sentinel mixture models},
  author={Merity, Stephen and Xiong, Caiming and Bradbury, James and Socher, Richard},
  journal={arXiv preprint arXiv:1609.07843},
  year={2016}
}

@article{krizhevsky2009learning,
  title={Learning multiple layers of features from tiny images},
  author={Krizhevsky, Alex and Hinton, Geoffrey and others},
  year={2009},
  publisher={Toronto, ON, Canada}
}

@article{paszke2019pytorch,
  title={Pytorch: An imperative style, high-performance deep learning library},
  author={Paszke, Adam and Gross, Sam and Massa, Francisco and Lerer, Adam and Bradbury, James and Chanan, Gregory and Killeen, Trevor and Lin, Zeming and Gimelshein, Natalia and Antiga, Luca and others},
  journal={Advances in neural information processing systems},
  volume={32},
  year={2019}
}

\appendix

\section{Proofs for Distance-Adaptive Muon}
\label{app:da_muon}

This appendix proves Theorem~\ref{thm_distance_adaptive}. We first collect the auxiliary summation estimates, then state the trust-region optimality identity used by Algorithm~\ref{algo_dis_adaptive}, and finally prove the momentum-tracking estimate and the stationarity theorem.

\subsection{Auxiliary summation estimates}

\begin{lemma}[Discrete convolution bound]
\label{lem_da_discrete_convolution}
Let \(\alpha\in(0,1)\), set \(q:=1-\alpha\), and let \(k\ge1\). Then
\[
    \sum_{i=0}^{k-1}
    \frac{q^{k-i}}{\sqrt{i+1}}
    \le
    \frac{q^{k/2}}{\alpha}
    +
    \frac{\sqrt{2}\,q}{\alpha\sqrt{k+2}} .
\]
\end{lemma}

\begin{proof}
Changing variables \(j=k-i\), we obtain
\[
    \sum_{i=0}^{k-1}
    \frac{q^{k-i}}{\sqrt{i+1}}
    =
    \sum_{j=1}^{k}
    \frac{q^j}{\sqrt{k-j+1}} .
\]
Split the sum into \(j\le k/2\) and \(j>k/2\). If \(j\le k/2\), then \(k-j+1\ge (k+2)/2\), and therefore
\[
    \sum_{1\le j\le k/2}
    \frac{q^j}{\sqrt{k-j+1}}
    \le
    \frac{\sqrt{2}}{\sqrt{k+2}}
    \sum_{j=1}^{\infty} q^j
    =
    \frac{\sqrt{2}\,q}{\alpha\sqrt{k+2}} .
\]
If \(j>k/2\), then \(1/\sqrt{k-j+1}\le1\), and hence
\[
    \sum_{k/2<j\le k}
    \frac{q^j}{\sqrt{k-j+1}}
    \le
    \sum_{j>k/2}^{\infty} q^j
    \le
    \frac{q^{k/2}}{\alpha}.
\]
Combining the two bounds proves the claim.
\end{proof}

\begin{lemma}
\label{lem_integration_upbd}
Let \(a\in(0,1)\) and \(t\ge 0\) be an integer. Then
\[
    \sum_{k=0}^t
    \frac{a^k}{\sqrt{k+1}}
    \le
    1+
    \frac{1}{a}
    \sqrt{
    \frac{\pi}{\log\!\left(\frac{1}{a}\right)}
    }.
\]
\end{lemma}

\begin{proof}
The function \(x\mapsto a^x/\sqrt{x+1}\) is nonincreasing on \([0,\infty)\). Hence
\[
    \sum_{k=0}^t
    \frac{a^k}{\sqrt{k+1}}
    \le
    1+
    \int_0^t
    \frac{a^k}{\sqrt{k+1}}\,dk
    =
    1+
    \int_0^t
    \frac{e^{k\log a}}{\sqrt{k+1}}\,dk .
\]
Since \(a<1\), we have \(\log a<0\). Define
\[
    \lambda:=-\log a>0.
\]
Then
\[
    \sum_{k=0}^t
    \frac{a^k}{\sqrt{k+1}}
    \le
    1+
    \int_0^\infty
    \frac{e^{-\lambda k}}{\sqrt{k+1}}\,dk .
\]
Applying the change of variables \(u=k+1\), we obtain
\begin{align*}
    \sum_{k=0}^t
    \frac{a^k}{\sqrt{k+1}}
    &\le
    1+
    e^\lambda
    \int_1^\infty
    \frac{e^{-\lambda u}}{\sqrt{u}}\,du  \\
    &\le
    1+
    e^\lambda
    \int_0^\infty
    \frac{e^{-\lambda u}}{\sqrt{u}}\,du
    =
    1+
    e^\lambda
    \sqrt{\frac{\pi}{\lambda}},
\end{align*}
where the last integral is the standard Gamma-function integral. Substituting \(\lambda=\log(1/a)\) and \(e^\lambda=1/a\) gives the desired bound.
\end{proof}

\begin{lemma}
\label{lem_rad_bd}
{\rm \cite[Lemma 30]{liu2023stochastic},
\cite[Lemma A.1]{moshtaghifar2025dada},
\cite[Lemma B.3]{chakraborty2026randomized}}
Let \(\{a_i\}_{i=0}^{T+1}\) be a positive, nondecreasing sequence. Then, for any integer \(T\ge1\),
\[
    \min_{1\le k\le T}
    \frac{a_{k+1}}{\sum_{i=1}^k a_i}
    \le
    \frac{1}{T}
    \left(
        \frac{a_{T+1}}{a_0}
    \right)^{1/T}
    \log\!\left(
        \frac{e\,a_{T+1}}{a_0}
    \right).
\]
\end{lemma}

\begin{lemma}[Tail radius-ratio bound]
\label{lem_rad_bd_tail}
Let \(\{\bar r_i\}_{i=0}^{T+1}\) be a positive, nondecreasing sequence. Then, for any integer \(T\ge2\),
\[
    \min_{\lfloor T/2\rfloor+1\le k\le T}
    \frac{\bar r_{k+1}}{\sum_{i=1}^{k}\bar r_i}
    \le
    \frac{2}{T}
    \left(
        \frac{\bar r_{T+1}}{\bar r_0}
    \right)^{2/T}
    \log\!\left(
        \frac{e\,\bar r_{T+1}}{\bar r_0}
    \right).
\]
\end{lemma}

\begin{proof}
Set \(s=\lfloor T/2\rfloor\) and \(M=T-s\). Define the shifted sequence
\[
    a_j:=\bar r_{s+j},
    \qquad
    j=0,\ldots,M+1 .
\]
Applying Lemma~\ref{lem_rad_bd} to \(\{a_j\}_{j=0}^{M+1}\) with horizon \(M\) gives
\[
    \min_{1\le j\le M}
    \frac{a_{j+1}}{\sum_{\ell=1}^{j}a_\ell}
    \le
    \frac{1}{M}
    \left(
        \frac{a_{M+1}}{a_0}
    \right)^{1/M}
    \log\!\left(
        \frac{e\,a_{M+1}}{a_0}
    \right).
\]
For \(k=s+j\), we have
\[
    \sum_{i=1}^{k}\bar r_i
    \ge
    \sum_{i=s+1}^{k}\bar r_i
    =
    \sum_{\ell=1}^{j}a_\ell,
    \qquad
    \bar r_{k+1}=a_{j+1}.
\]
Therefore,
\[
    \min_{s+1\le k\le T}
    \frac{\bar r_{k+1}}{\sum_{i=1}^{k}\bar r_i}
    \le
    \frac{1}{M}
    \left(
        \frac{\bar r_{T+1}}{\bar r_s}
    \right)^{1/M}
    \log\!\left(
        \frac{e\,\bar r_{T+1}}{\bar r_s}
    \right).
\]
Since \(M\ge T/2\), \(\bar r_s\ge \bar r_0\), and \(\bar r_{T+1}\ge \bar r_0\), the right-hand side is at most
\[
    \frac{2}{T}
    \left(
        \frac{\bar r_{T+1}}{\bar r_0}
    \right)^{2/T}
    \log\!\left(
        \frac{e\,\bar r_{T+1}}{\bar r_0}
    \right).
\]
This proves the claim.
\end{proof}

\subsection{Trust-region optimality identity}

\begin{lemma}[Trust-region optimality identity]
\label{lem_opt}
{\rm \cite[Modification of Lemma 3]{kovalev2025understanding}.}
Let \(\eta\ge0\), let \(\Psi:\mathcal E\to\mathbb R\) be differentiable, and let \(z_+\in\mathcal E\) be a solution of
\[
    \min_{x\in\mathcal E} \Psi(x)
    \qquad
    \text{s.t.}
    \qquad
    \|x-z\|\le\eta .
\]
Then
\[
    \langle \nabla \Psi(z_+),z-z_+\rangle
    =
    \eta\|\nabla \Psi(z_+)\|_* .
\]
\end{lemma}

\begin{proof}
Let
\[
    B_\eta(z):=\{x\in\mathcal E:\|x-z\|\le\eta\}.
\]
If \(\eta=0\), then \(z_+=z\), and the claim is immediate. Hence, assume \(\eta>0\). Since \(z_+\) minimizes \(\Psi\) over \(B_\eta(z)\), for every \(x\in B_\eta(z)\) and every \(\tau\in[0,1]\), the point \(z_+ + \tau(x-z_+)\) remains feasible. Therefore, the one-sided directional derivative at \(\tau=0\) is nonnegative:
\[
    \langle \nabla\Psi(z_+),x-z_+\rangle \ge 0,
    \qquad
    \forall x\in B_\eta(z).
\]
If \(\nabla\Psi(z_+)=0\), the claimed identity is immediate. Otherwise, \(z_+\) cannot lie in the interior of \(B_\eta(z)\), and hence
\[
    \|z_+-z\|=\eta .
\]
Set \(g:=\nabla\Psi(z_+)\) and \(w:=z_+-z\). The variational inequality above implies that, for all \(d\) with \(\|d\|\le\eta\),
\[
    \langle g,d-w\rangle\ge0,
    \qquad\text{or equivalently}\qquad
    \langle g,w\rangle\le \langle g,d\rangle .
\]
Taking the infimum over \(\|d\|\le\eta\), we obtain
\[
    \langle g,w\rangle
    \le
    \inf_{\|d\|\le\eta}\langle g,d\rangle
    =
    -\eta\|g\|_* .
\]
On the other hand, by H\"older's inequality and \(\|w\|=\eta\),
\[
    \langle g,w\rangle
    \ge
    -\|g\|_*\|w\|
    =
    -\eta\|g\|_* .
\]
Thus
\[
    \langle g,w\rangle=-\eta\|g\|_* .
\]
Since \(w=z_+-z\), this is equivalent to
\[
    \langle \nabla\Psi(z_+),z-z_+\rangle
    =
    \eta\|\nabla\Psi(z_+)\|_* ,
\]
as claimed.
\end{proof}

\subsection{Momentum tracking}

\begin{lemma}[Momentum tracking]
\label{lem_momentum_bd}
Let Assumption~\ref{assum:smoothness} hold, and consider the recursion
\[
    m_{k+1}
    =
    (1-\alpha)m_k+\alpha\nabla f(x_k),
    \qquad
    \alpha\in(0,1).
\]
Assume that the iterates satisfy
\[
    \|x_{i+1}-x_i\|\le \eta_i .
\]
Then, for all \(k\ge0\), with the summation interpreted as zero when \(k=0\),
\[
\|m_{k+1}-\nabla f(x_k)\|_*
\le
(1-\alpha)^{k+1}
\|m_0-\nabla f(x_0)\|_*
+
L\sum_{i=0}^{k-1}
(1-\alpha)^{k-i}\eta_i .
\]
Moreover, if \(m_0=\nabla f(x_0)\) and
\[
    \eta_i=\frac{\bar r_{i+1}}{\sqrt{i+1}},
\]
then, for all \(k\ge0\),
\[
    \|m_{k+1}-\nabla f(x_k)\|_*
    \le
    L\bar r_k
    \left[
        \frac{(1-\alpha)^{k/2}}{\alpha}
        +
        \frac{\sqrt{2}(1-\alpha)}
        {\alpha\sqrt{k+2}}
    \right].
\]
\end{lemma}

\begin{proof}
For \(k=0\), the first claim follows directly from
\[
    m_1-\nabla f(x_0)
    =
    (1-\alpha)(m_0-\nabla f(x_0)).
\]
If, in addition, \(m_0=\nabla f(x_0)\), then the left-hand side of the second claim is zero for \(k=0\), so the second claim also holds. We now consider \(k\ge1\).

We start from the recursion
\[
m_{k+1}-\nabla f(x_k)
=
(1-\alpha)m_k+\alpha\nabla f(x_k)-\nabla f(x_k),
\]
which can be rewritten as
\[
m_{k+1}-\nabla f(x_k)
=
(1-\alpha)
\bigl(m_k-\nabla f(x_{k-1})\bigr)
+
(1-\alpha)
\bigl(\nabla f(x_{k-1})-\nabla f(x_k)\bigr).
\]
Unrolling the recursion yields
\[
m_{k+1}-\nabla f(x_k)
=
(1-\alpha)^{k+1}
(m_0-\nabla f(x_0))
+
\sum_{i=0}^{k-1}
(1-\alpha)^{k-i}
\bigl(\nabla f(x_i)-\nabla f(x_{i+1})\bigr).
\]
Taking dual norms and applying the triangle inequality yields
\[
\|m_{k+1}-\nabla f(x_k)\|_*
\le
(1-\alpha)^{k+1}
\|m_0-\nabla f(x_0)\|_*
+
\sum_{i=0}^{k-1}
(1-\alpha)^{k-i}
\|\nabla f(x_i)-\nabla f(x_{i+1})\|_* .
\]
Using \(L\)-smoothness of \(f\), we obtain
\[
\|m_{k+1}-\nabla f(x_k)\|_*
\le
(1-\alpha)^{k+1}
\|m_0-\nabla f(x_0)\|_*
+
L
\sum_{i=0}^{k-1}
(1-\alpha)^{k-i}
\|x_i-x_{i+1}\|.
\]
From the trust-region bound of Algorithm~\ref{algo_dis_adaptive}, we have
\[
    \|x_{i+1}-x_i\|\le \eta_i,
\]
which proves the first claim.

For the case \(m_0=\nabla f(x_0)\) and
\[
    \eta_i=\frac{\bar r_{i+1}}{\sqrt{i+1}},
\]
the first claim simplifies to
\[
    \|m_{k+1}-\nabla f(x_k)\|_*
    \le
    L
    \sum_{i=0}^{k-1}
    (1-\alpha)^{k-i}
    \frac{\bar r_{i+1}}{\sqrt{i+1}}
    \le
    L\bar r_k
    \sum_{i=0}^{k-1}
    \frac{(1-\alpha)^{k-i}}{\sqrt{i+1}}.
\]
Applying Lemma~\ref{lem_da_discrete_convolution} with \(q=1-\alpha\), we obtain
\[
    \|m_{k+1}-\nabla f(x_k)\|_*
    \le
    L\bar r_k
    \left[
        \frac{(1-\alpha)^{k/2}}{\alpha}
        +
        \frac{\sqrt{2}(1-\alpha)}
        {\alpha\sqrt{k+2}}
    \right].
\]
This proves the second claim.
\end{proof}

\subsection{Proof of Theorem~\ref{thm_distance_adaptive}}

\begin{proof}
Using \(L\)-smoothness of \(f\), we obtain
\begin{align}
    f(x_{k+1})
    \le
    f(x_k)
    +
    \langle \nabla f(x_k),x_{k+1}-x_k\rangle
    +
    \frac{L}{2}\|x_{k+1}-x_k\|^2 .
    \label{eq:da_smooth_start}
\end{align}
The second term on the right-hand side can be split as
\begin{align*}
    \langle \nabla f(x_k),x_{k+1}-x_k\rangle
    &=
    \langle m_{k+1},x_{k+1}-x_k\rangle \\
    &\quad+
    \langle \nabla f(x_{k+1})-m_{k+1},x_{k+1}-x_k\rangle \\
    &\quad+
    \langle \nabla f(x_k)-\nabla f(x_{k+1}),x_{k+1}-x_k\rangle .
\end{align*}
We apply H\"older's inequality to the last two terms on the right-hand side and then apply Assumption~\ref{assum:smoothness} to obtain
\begin{align*}
    \langle \nabla f(x_k),x_{k+1}-x_k\rangle
    &\le
    \langle m_{k+1},x_{k+1}-x_k\rangle
    +
    \|\nabla f(x_{k+1})-m_{k+1}\|_*
    \|x_{k+1}-x_k\| \\
    &\quad+
    L\|x_{k+1}-x_k\|^2 .
\end{align*}
Substituting the preceding relation in \eqref{eq:da_smooth_start} and using \(\|x_{k+1}-x_k\|\le\eta_k\), we get
\[
    f(x_{k+1})
    \le
    f(x_k)
    +
    \langle m_{k+1},x_{k+1}-x_k\rangle
    +
    \eta_k\|\nabla f(x_{k+1})-m_{k+1}\|_*
    +
    \frac{3L}{2}\eta_k^2 .
\]
By construction, \(x_{k+1}\) solves the linear trust-region problem
\[
    \min_{\|x-x_k\|\le\eta_k}
    \langle m_{k+1},x\rangle .
\]
Indeed, Algorithm~\ref{algo_dis_adaptive} chooses \(u_{k+1}\in\argmax_{\|u\|\le1}\langle m_{k+1},u\rangle\) and sets \(x_{k+1}=x_k-\eta_k u_{k+1}\). Applying Lemma~\ref{lem_opt} with \(\Psi(x)=\langle m_{k+1},x\rangle\), \(z=x_k\), \(z_+=x_{k+1}\), and \(\eta=\eta_k\), we obtain
\[
    \langle m_{k+1},x_{k+1}-x_k\rangle
    =
    -\eta_k\|m_{k+1}\|_* .
\]
Therefore,
\[
    f(x_{k+1})
    \le
    f(x_k)
    -
    \eta_k\|m_{k+1}\|_*
    +
    \eta_k\|\nabla f(x_{k+1})-m_{k+1}\|_*
    +
    \frac{3L}{2}\eta_k^2 .
\]
Applying the triangle inequality
\[
    -\|m_{k+1}\|_*
    \le
    \|\nabla f(x_{k+1})-m_{k+1}\|_*
    -
    \|\nabla f(x_{k+1})\|_*
\]
gives
\[
    f(x_{k+1})
    \le
    f(x_k)
    -
    \eta_k\|\nabla f(x_{k+1})\|_*
    +
    2\eta_k\|\nabla f(x_{k+1})-m_{k+1}\|_*
    +
    \frac{3L}{2}\eta_k^2 .
\]
We apply the triangle inequality to the third term on the right-hand side:
\begin{align*}
    f(x_{k+1})
    &\le
    f(x_k)
    -
    \eta_k\|\nabla f(x_{k+1})\|_*
    +
    2\eta_k\|\nabla f(x_k)-m_{k+1}\|_* \\
    &\quad+
    2\eta_k\|\nabla f(x_{k+1})-\nabla f(x_k)\|_*
    +
    \frac{3L}{2}\eta_k^2 \\
    &\le
    f(x_k)
    -
    \eta_k\|\nabla f(x_{k+1})\|_*
    +
    2\eta_k\|\nabla f(x_k)-m_{k+1}\|_* \\
    &\quad+
    2\eta_k L\|x_{k+1}-x_k\|
    +
    \frac{3L}{2}\eta_k^2 .
\end{align*}
Since Algorithm~\ref{algo_dis_adaptive} sets \(m_0=\nabla f(x_0)\) and
\[
    \eta_i=\frac{\bar r_{i+1}}{\sqrt{i+1}},
\]
the third and fourth quantities on the right-hand side can be bounded using Lemma~\ref{lem_momentum_bd} and the trust-region bound \(\|x_{k+1}-x_k\|\le\eta_k\). This yields
\begin{align*}
    f(x_{k+1})
    \le
    f(x_k)
    &-
    \eta_k\|\nabla f(x_{k+1})\|_* \\
    &+
    \frac{2\eta_k L\bar r_k}{\alpha}
    (1-\alpha)^{k/2} \\
    &+
    \frac{2\sqrt{2}(1-\alpha)\eta_k L}
    {\alpha}
    \frac{\bar r_k}{\sqrt{k+2}}
    +
    \frac{7L}{2}\eta_k^2 .
\end{align*}
Note that
\[
    \frac{\bar r_k}{\sqrt{k+2}}
    \le
    \frac{\bar r_{k+1}}{\sqrt{k+1}}
    =
    \eta_k .
\]
Hence the preceding relation simplifies to
\begin{align*}
    f(x_{k+1})
    \le
    f(x_k)
    &-
    \eta_k\|\nabla f(x_{k+1})\|_* \\
    &+
    \frac{2\eta_k L\bar r_k}{\alpha}
    (1-\alpha)^{k/2} \\
    &+
    \left[
        \frac{7}{2}
        +
        \frac{2\sqrt{2}(1-\alpha)}{\alpha}
    \right]
    L\eta_k^2 .
\end{align*}
Rearranging terms and summing both sides from \(k=0\) to \(t\), for all \(t\ge0\), gives
\begin{align}
    \sum_{k=0}^t
    \eta_k\|\nabla f(x_{k+1})\|_*
    \le
    f(x_0)-f(x_{t+1})
    &+
    \frac{2L}{\alpha}
    \sum_{k=0}^t
    \bar r_k\eta_k
    (1-\alpha)^{k/2}
    \nonumber\\
    &+
    \left[
        \frac{7}{2}
        +
        \frac{2\sqrt{2}(1-\alpha)}{\alpha}
    \right]
    L
    \sum_{k=0}^t
    \eta_k^2 .
    \label{eq:da_main_sum}
\end{align}
The left-hand side of \eqref{eq:da_main_sum} can be lower-bounded as
\begin{align}
    \sum_{k=0}^t
    \eta_k\|\nabla f(x_{k+1})\|_*
    &\ge
    \min_{0\le k\le t}
    \|\nabla f(x_{k+1})\|_*
    \sum_{k=0}^t
    \frac{\bar r_{k+1}}{\sqrt{k+1}}  \nonumber\\
    &\ge
    \frac{
    \min_{0\le k\le t}
    \|\nabla f(x_{k+1})\|_*
    }{\sqrt{t+1}}
    \sum_{k=0}^t
    \bar r_{k+1}.
    \label{eq:da_lhs_lower}
\end{align}
The momentum-tracking term on the right-hand side of \eqref{eq:da_main_sum} can be upper-bounded using Lemma~\ref{lem_integration_upbd} with \(a=\sqrt{1-\alpha}\):
\begin{align}
    \frac{2L}{\alpha}
    \sum_{k=0}^t
    \bar r_k\eta_k
    (1-\alpha)^{k/2}
    &\le
    \frac{2L\bar r_{t+1}^2}{\alpha}
    \sum_{k=0}^t
    \frac{(1-\alpha)^{k/2}}{\sqrt{k+1}}
    \nonumber\\
    &\le
    2L\bar r_{t+1}^2 A(\alpha),
    \label{eq:da_momentum_term}
\end{align}
where
\[
    A(\alpha)
    =
    \frac{1}{\alpha}
    \left[
        1
        +
        \frac{1}{\sqrt{1-\alpha}}
        \sqrt{
        \frac{2\pi}
        {\log\!\left(\frac{1}{1-\alpha}\right)}
        }
    \right].
\]
The quadratic term on the right-hand side of \eqref{eq:da_main_sum} can be upper-bounded as
\begin{align}
    \sum_{k=0}^t\eta_k^2
    \le
    \bar r_{t+1}^2
    \sum_{k=0}^t
    \frac{1}{k+1}
    \le
    \bar r_{t+1}^2(1+\log(t+1)).
    \label{eq:da_quadratic_term}
\end{align}
Using \( -f(x_{t+1})\le -f^\star\), and applying \eqref{eq:da_lhs_lower}, \eqref{eq:da_momentum_term}, and \eqref{eq:da_quadratic_term} to \eqref{eq:da_main_sum}, we get
\begin{align}
    \min_{0\le k\le t}
    \|\nabla f(x_{k+1})\|_*
    &\le
    \frac{\sqrt{t+1}(f(x_0)-f^\star)}
    {\sum_{k=0}^t \bar r_{k+1}}
    +
    L C_{\alpha}(t+1)
    \frac{\sqrt{t+1}\bar r_{t+1}^2}
    {\sum_{k=0}^t \bar r_{k+1}}  \nonumber\\
    &\le
    \frac{f(x_0)-f^\star}{r\sqrt{t+1}}
    +
    L C_{\alpha}(t+1)
    \frac{\sqrt{t+1}\bar r_{t+1}^2}
    {\sum_{k=0}^t \bar r_{k+1}},
    \label{eq:da_pre_radius}
\end{align}
where we used \(r\le \bar r_{k+1}\) for all \(k\ge0\), and
\[
    C_{\alpha}(t+1)
    :=
    2A(\alpha)
    +
    \left[
        \frac{7}{2}
        +
        \frac{2\sqrt{2}(1-\alpha)}{\alpha}
    \right]
    (1+\log(t+1)).
\]

For notational convenience, extend the radius sequence by one final update
\[
    \bar r_{T+1}:=\max\{\bar r_T,\|x_T-x_0\|\}.
\]
Since the trajectory bound is used only in the analysis, we may enlarge it and assume without loss of generality that \(D\ge r\). Under Assumption~\ref{assum:da_bounded_trajectory}, this gives
\[
    \bar r_{T+1}\le D .
\]

If \(T=1\), then \eqref{eq:da_pre_radius} with \(t=0\) gives
\[
    \min_{0\le k\le 0}
    \|\nabla f(x_{k+1})\|_*
    \le
    \frac{f(x_0)-f^\star}{r}
    +
    L C_\alpha(1)D,
\]
which is bounded by the claimed estimate because \(D/r\ge1\) and \(\log(eD/r)\ge1\). Hence, it remains to consider \(T\ge2\).

Set \(s=\lfloor T/2\rfloor\). For every \(k\in\{s+1,\ldots,T\}\), applying \eqref{eq:da_pre_radius} with \(t=k-1\) gives
\begin{align}
    \min_{0\le j\le k-1}
    \|\nabla f(x_{j+1})\|_*
    &\le
    \frac{f(x_0)-f^\star}{r\sqrt{k}}
    +
    L C_{\alpha}(k)
    \frac{\sqrt{k}\bar r_k^2}
    {\sum_{i=1}^{k}\bar r_i}.
    \label{eq:da_tail_pre_radius}
\end{align}
Since \(k\ge s+1>T/2\), \(k\le T\), \(C_\alpha(k)\le C_\alpha(T)\), and \(\bar r_k^2\le D\bar r_{k+1}\), we obtain
\[
    \min_{0\le j\le k-1}
    \|\nabla f(x_{j+1})\|_*
    \le
    \frac{\sqrt{2}\bigl(f(x_0)-f^\star\bigr)}{r\sqrt T}
    +
    L C_{\alpha}(T)\sqrt T\,D
    \frac{\bar r_{k+1}}{\sum_{i=1}^{k}\bar r_i}.
\]
Moreover,
\[
    \min_{0\le j\le T-1}
    \|\nabla f(x_{j+1})\|_*
    \le
    \min_{0\le j\le k-1}
    \|\nabla f(x_{j+1})\|_*
\]
for every \(k\le T\). Taking the minimum over \(k\in\{s+1,\ldots,T\}\), we get
\begin{align}
    \min_{0\le j\le T-1}
    \|\nabla f(x_{j+1})\|_*
    &\le
    \frac{\sqrt{2}\bigl(f(x_0)-f^\star\bigr)}{r\sqrt T}
    +
    L C_{\alpha}(T)\sqrt T\,D
    \min_{s+1\le k\le T}
    \frac{\bar r_{k+1}}{\sum_{i=1}^{k}\bar r_i}.
    \label{eq:da_before_tail_radius_lemma}
\end{align}
Applying Lemma~\ref{lem_rad_bd_tail} to the second term on the right-hand side of \eqref{eq:da_before_tail_radius_lemma} yields
\begin{align*}
    \min_{0\le j\le T-1}
    \|\nabla f(x_{j+1})\|_*
    &\le
    \frac{\sqrt{2}\bigl(f(x_0)-f^\star\bigr)}{r\sqrt T}
    +
    \frac{2L C_{\alpha}(T)D}{\sqrt T}
    \left(
        \frac{\bar r_{T+1}}{\bar r_0}
    \right)^{2/T}
    \log\!\left(
        \frac{e\,\bar r_{T+1}}{\bar r_0}
    \right) \\
    &\le
    \frac{\sqrt{2}\bigl(f(x_0)-f^\star\bigr)}{r\sqrt T}
    +
    \frac{2L C_{\alpha}(T)D}{\sqrt T}
    \left(
        \frac{D}{r}
    \right)^{2/T}
    \log\!\left(
        \frac{eD}{r}
    \right),
\end{align*}
where we used \(\bar r_{T+1}\le D\) and \(\bar r_0=r\). Hence, the claim follows.
\end{proof}
\section{Proofs for Scale-Calibrated Muon}
\label{app:sc_muon}

This appendix proves Theorem~\ref{thm:certified_rate}. We use the notation
\[
    \Delta_k:=f(x_k)-f(x^\star),
    \qquad
    q:=1-\alpha,
    \qquad
    \gamma:=1-q^2=\alpha(2-\alpha).
\]
Define
\[
    A:=\frac{\gamma}{16q^2},
    \qquad
    \chi_\alpha:=\min\left\{\frac{3}{8},\frac{\gamma^3}{16q^2}\right\}>0,
\]
and the Lyapunov function
\[
    \Phi_k:=\Delta_k+\frac{A}{L}e_k^2.
\]
Here \(\chi_\alpha\) is a Lyapunov decrease constant and should not be confused with a strong-convexity parameter.

\subsection{Basic descent and error estimates}

\begin{lemma}[Certified descent]
\label{lem:cert_descent}
For every \(k\ge0\), Algorithm~\ref{alg:scale_calibrated_muon} satisfies
\[
    f(x_{k+1})\le f(x_k)-\frac{a_k^2}{2L}.
\]
Consequently, the objective values are nonincreasing.
\end{lemma}

\begin{proof}
If \(a_k=0\), then \(\eta_k=0\) and \(x_{k+1}=x_k\). Suppose \(a_k>0\). Then \(a_k=\|m_{k+1}\|_*-e_k\), and the choice of \(u_k\) gives
\[
\begin{aligned}
    \langle g_k,u_k\rangle
    &=\langle m_{k+1},u_k\rangle+\langle g_k-m_{k+1},u_k\rangle \\
    &\ge \|m_{k+1}\|_*-e_k=a_k.
\end{aligned}
\]
By \(L\)-smoothness and \(\|u_k\|\le1\),
\[
\begin{aligned}
    f(x_{k+1})
    &\le f(x_k)-\eta_k\langle g_k,u_k\rangle
    +\frac{L}{2}\eta_k^2\|u_k\|^2 \\
    &\le f(x_k)-\eta_k a_k+\frac{L}{2}\eta_k^2.
\end{aligned}
\]
Substituting \(\eta_k=a_k/L\) proves the claim.
\end{proof}

\begin{lemma}[Momentum-error recursion]
\label{lem:cert_error_recursion}
Let \(q:=1-\alpha\). Then
\[
    e_{k+1}\le q(e_k+a_k)
    \qquad \forall k\ge0.
\]
\end{lemma}

\begin{proof}
Since \(m_{k+2}=q m_{k+1}+\alpha g_{k+1}\), we have
\[
    g_{k+1}-m_{k+2}=q(g_{k+1}-m_{k+1}).
\]
Therefore, using smoothness and \(\|x_{k+1}-x_k\|\le \eta_k=a_k/L\),
\[
\begin{aligned}
    e_{k+1}
    &=\|g_{k+1}-m_{k+2}\|_* \\
    &\le q\|g_{k+1}-g_k\|_*+q\|g_k-m_{k+1}\|_* \\
    &\le qL\|x_{k+1}-x_k\|+q e_k
    \le q(a_k+e_k).
\end{aligned}
\]
\end{proof}

\begin{lemma}[Gap-to-certificate relation]
\label{lem:cert_gap_relation}
Let \(\Delta_k:=f(x_k)-f(x^\star)\). Under Assumptions~\ref{assum:star_convex} and~\ref{assum:bounded_level}, the iterates of Algorithm~\ref{alg:scale_calibrated_muon} satisfy
\[
    \Delta_k\le D_{\rm lev}(a_k+2e_k)
    \qquad \forall k\ge0.
\]
\end{lemma}

\begin{proof}
By Lemma~\ref{lem:cert_descent}, \(f(x_k)\le f(x_0)\) for all \(k\). Hence Assumption~\ref{assum:bounded_level} gives
\[
    \|x_k-x^\star\|\le D_{\rm lev}.
\]
Star-convexity implies
\[
    \Delta_k
    \le \langle g_k,x_k-x^\star\rangle
    \le D_{\rm lev}\|g_k\|_*.
\]
If \(a_k>0\), then \(a_k=\|m_{k+1}\|_*-e_k\), hence
\[
    \|m_{k+1}\|_*=a_k+e_k
\]
and
\[
    \|g_k\|_*
    \le
    \|m_{k+1}\|_*+\|g_k-m_{k+1}\|_*
    =
    a_k+2e_k.
\]
If \(a_k=0\), then \(\|m_{k+1}\|_*\le e_k\), so
\[
    \|g_k\|_*\le 2e_k=a_k+2e_k.
\]
Combining these inequalities proves the result.
\end{proof}

\subsection{Lyapunov analysis}

The objective gap alone may not decrease rapidly when the momentum is stale. The Lyapunov function \(\Phi_k\) combines the objective gap with the momentum-tracking error, allowing the descent generated by the certificate \(a_k\) to amortize the stale momentum error.

\begin{lemma}[Lyapunov descent]
\label{lem:cert_lyap_descent}
For all \(k\ge0\),
\[
    \Phi_{k+1}\le \Phi_k-\frac{\chi_\alpha}{L}(a_k^2+e_k^2).
\]
\end{lemma}

\begin{proof}
By Lemmas~\ref{lem:cert_descent} and~\ref{lem:cert_error_recursion},
\[
\begin{aligned}
    \Phi_k-\Phi_{k+1}
    &\ge
    \frac{a_k^2}{2L}
    -
    \frac{A}{L}
    \left(q^2(e_k+a_k)^2-e_k^2\right) \\
    &=
    \frac{1}{L}
    \left[
        \left(\frac12-Aq^2\right)a_k^2
        -2Aq^2a_ke_k
        +A(1-q^2)e_k^2
    \right].
\end{aligned}
\]
Using \(A=\gamma/(16q^2)\) and \(\gamma=1-q^2\), the bracket equals
\[
    \left(\frac12-\frac{\gamma}{16}\right)a_k^2
    -\frac{\gamma}{8}a_ke_k
    +\frac{\gamma^2}{16q^2}e_k^2.
\]
Since \(\gamma\le1\), the first coefficient is at least \(7/16\). By Young's inequality,
\[
    \frac{\gamma}{8}a_ke_k
    \le
    \frac{1}{16}a_k^2+\frac{\gamma^2}{16}e_k^2.
\]
Thus the bracket is at least
\[
    \frac{3}{8}a_k^2+\frac{\gamma^3}{16q^2}e_k^2
    \ge
    \chi_\alpha(a_k^2+e_k^2),
\]
which proves the lemma.
\end{proof}

\begin{lemma}[Potential control]
\label{lem:cert_phi_control}
For all \(k\ge0\),
\[
    \Phi_k\le C_{{\rm sc},\alpha}D_{\rm lev}\sqrt{a_k^2+e_k^2},
    \qquad
    C_{{\rm sc},\alpha}:=\sqrt5+2A.
\]
\end{lemma}

\begin{proof}
Lemma~\ref{lem:cert_gap_relation} gives
\[
    \Delta_k
    \le
    D_{\rm lev}(a_k+2e_k)
    \le
    \sqrt5\,D_{\rm lev}\sqrt{a_k^2+e_k^2}.
\]
It remains to bound \(Ae_k^2/L\). Since \(f(x_k)\le f(x_0)\), Assumption~\ref{assum:bounded_level} gives
\[
    \|x_k-x^\star\|\le D_{\rm lev}.
\]
Because we are in the unconstrained differentiable setting,
\[
    \nabla f(x^\star)=0.
\]
Smoothness gives
\[
    \|g_k\|_*\le LD_{\rm lev}.
\]
Moreover, \(m_{k+1}\) is a convex combination of past gradients and all past iterates remain in the same initial sublevel set, so
\[
    \|m_{k+1}\|_*\le LD_{\rm lev}.
\]
Therefore
\[
    e_k\le 2LD_{\rm lev},
\]
and hence
\[
    \frac{e_k^2}{L}
    \le
    2D_{\rm lev}e_k
    \le
    2D_{\rm lev}\sqrt{a_k^2+e_k^2}.
\]
Combining the two estimates yields the claim.
\end{proof}

\subsection{Refined Lyapunov guarantee}

\begin{theorem}[Refined Lyapunov guarantee for Scale-Calibrated Muon]
\label{thm:sc_muon_refined}
Let Assumptions~\ref{assum:smoothness}, \ref{assum:star_convex}, and \ref{assum:bounded_level} hold, and consider Algorithm~\ref{alg:scale_calibrated_muon} with \(\alpha\in(0,1)\). Then
\[
    \Delta_T\le \Phi_T
    \qquad
    \text{for all } T\ge0.
\]
Moreover, if \(\Phi_0=0\), then \(\Phi_T=0\) for all \(T\ge0\). Otherwise, for all \(T\ge0\),
\[
    \Phi_T
    \le
    \frac{1}{
        \Phi_0^{-1}
        +
        \dfrac{\chi_\alpha T}
        {C_{{\rm sc},\alpha}^2LD_{\rm lev}^2}
    },
    \qquad
    \Phi_0=\Delta_0.
\]
In particular, for all \(T\ge1\),
\[
    f(x_T)-f(x^\star)
    \le
    \frac{C_{{\rm sc},\alpha}^2}{\chi_\alpha}
    \frac{LD_{\rm lev}^2}{T}.
\]
\end{theorem}

\begin{proof}
By Lemma~\ref{lem:cert_phi_control},
\[
    a_k^2+e_k^2
    \ge
    \frac{\Phi_k^2}{C_{{\rm sc},\alpha}^2D_{\rm lev}^2}.
\]
Combining this with Lemma~\ref{lem:cert_lyap_descent},
\[
    \Phi_{k+1}
    \le
    \Phi_k
    -
    \frac{\chi_\alpha}
    {C_{{\rm sc},\alpha}^2LD_{\rm lev}^2}
    \Phi_k^2.
\]
Let
\[
    \theta_\alpha
    :=
    \frac{\chi_\alpha}
    {C_{{\rm sc},\alpha}^2LD_{\rm lev}^2}.
\]
If \(\Phi_k=0\), then the preceding recursion and nonnegativity of \(\Phi_{k+1}\) imply \(\Phi_{k+1}=0\). Thus, if \(\Phi_0=0\), then \(\Phi_T=0\) for all \(T\ge0\).

Now suppose \(\Phi_0>0\). As long as \(\Phi_k>0\), the recursion gives
\[
    \Phi_{k+1}
    \le
    \Phi_k(1-\theta_\alpha\Phi_k).
\]
Since \(\Phi_{k+1}\ge0\), we have \(\theta_\alpha\Phi_k\le1\), and therefore
\[
    \frac{1}{\Phi_{k+1}}
    \ge
    \frac{1}{\Phi_k(1-\theta_\alpha\Phi_k)}
    \ge
    \frac{1}{\Phi_k}+\theta_\alpha.
\]
Summing from \(k=0\) to \(T-1\) yields
\[
    \frac{1}{\Phi_T}
    \ge
    \frac{1}{\Phi_0}+\theta_\alpha T.
\]
Equivalently,
\[
    \Phi_T
    \le
    \frac{1}{
        \Phi_0^{-1}
        +
        \dfrac{\chi_\alpha T}
        {C_{{\rm sc},\alpha}^2LD_{\rm lev}^2}
    }.
\]
Finally, since \(m_0=g_0\), we have
\[
    m_1=(1-\alpha)m_0+\alpha g_0=g_0,
\]
so \(e_0=0\) and hence \(\Phi_0=\Delta_0\). Also,
\[
    \Delta_T\le\Phi_T
\]
by definition of \(\Phi_T\). For \(T\ge1\), the displayed bound implies
\[
    f(x_T)-f(x^\star)
    =
    \Delta_T
    \le
    \frac{C_{{\rm sc},\alpha}^2}{\chi_\alpha}
    \frac{LD_{\rm lev}^2}{T}.
\]
This proves the theorem.
\end{proof}

\begin{proof}[Proof of Theorem~\ref{thm:certified_rate}]
Theorem~\ref{thm:sc_muon_refined} gives, for all \(T\ge1\),
\[
    f(x_T)-f(x^\star)
    \le
    \frac{C_{{\rm sc},\alpha}^2}{\chi_\alpha}
    \frac{LD_{\rm lev}^2}{T}.
\]
Thus Theorem~\ref{thm:certified_rate} holds with
\[
    K_\alpha:=\frac{C_{{\rm sc},\alpha}^2}{\chi_\alpha}.
\]
The distance parameter \(D_{\rm lev}\) appears only in the analysis through Assumption~\ref{assum:bounded_level}; it is not used by Algorithm~\ref{alg:scale_calibrated_muon}.
\end{proof}
\section{Proofs for Distance-Free Muon}
\label{app:df_muon}

This appendix proves Theorem~\ref{thm:df_muon}. Throughout the section, let
\[
    f^\star:=f(x^\star),
    \qquad
    D:=\|x_0-x^\star\|,
    \qquad
    y_k:=x_k-x_0,
    \qquad
    u^\star:=x^\star-x_0 .
\]
Thus \(D=\|u^\star\|\). We also write
\[
    \Delta_k:=f(x_k)-f^\star .
\]

\subsection{Validity of the scalar D-certificate}

The scalar proxy \(d_k\) is inspired by D-adaptation, but the proof below only uses the one-sided lower-certificate property \(d_k\le D\).

\begin{lemma}[Validity of the D-certificate]
\label{lem:df_muon_dcert}
Suppose Assumption~\ref{assum:star_convex} holds and \(0\le d_0\le D\). Then the sequence generated by Algorithm~\ref{alg:pf_tr_muon} satisfies
\[
    d_k\le D
    \qquad
    \text{for all } k\ge0.
\]
\end{lemma}

\begin{proof}
By Assumption~\ref{assum:star_convex}, for every \(i\),
\[
    f(x_i)-f^\star
    \le
    \langle \nabla f(x_i),x_i-x^\star\rangle .
\]
Since \(x_i-x^\star=(x_i-x_0)-(x^\star-x_0)=y_i-u^\star\), we obtain
\[
    f(x_i)-f^\star
    \le
    \langle \nabla f(x_i),y_i\rangle
    -
    \langle \nabla f(x_i),u^\star\rangle .
\]
Multiplying by \(\omega_i\ge0\) and summing from \(i=0\) to \(k\) gives
\[
    \sum_{i=0}^k\omega_i(f(x_i)-f^\star)
    \le
    \sum_{i=0}^k\omega_i
    \langle \nabla f(x_i),y_i\rangle
    -
    \left\langle
        \sum_{i=0}^k\omega_i\nabla f(x_i),
        u^\star
    \right\rangle .
\]
The left-hand side is nonnegative because \(x^\star\) is a global minimizer. Using the definitions
\[
    S_{k+1}:=\sum_{i=0}^k\omega_i\nabla f(x_i),
    \qquad
    B_{k+1}:=
    -\sum_{i=0}^k\omega_i
    \langle\nabla f(x_i),x_i-x_0\rangle ,
\]
we therefore have
\[
    0
    \le
    -B_{k+1}
    -
    \langle S_{k+1},u^\star\rangle .
\]
Equivalently,
\[
    B_{k+1}
    \le
    -\langle S_{k+1},u^\star\rangle
    \le
    \|S_{k+1}\|_*\,\|u^\star\|
    =
    D\|S_{k+1}\|_*.
\]
If \(\|S_{k+1}\|_*>0\), then
\[
    \widehat d_{k+1}
    =
    \frac{[B_{k+1}]_+}{\|S_{k+1}\|_*}
    \le D.
\]
If \(\|S_{k+1}\|_*=0\), then by definition \(\widehat d_{k+1}=0\le D\). Since \(d_{k+1}=\max\{d_k,\widehat d_{k+1}\}\) and \(d_0\le D\), induction gives the claim.
\end{proof}

\subsection{Existence and convexity of the scalar search}

\begin{lemma}[Existence and convexity of the majorized scalar search]
\label{lem:df_muon_existence}
For every \(k\), the scalar minimization problem defining \(R_k\) in Algorithm~\ref{alg:pf_tr_muon} admits a solution. Moreover, \(R\mapsto\mathcal Q_k(R)\) is convex on \([0,\infty)\).
\end{lemma}

\begin{proof}
If \(L=0\), then by smoothness \(\nabla f\) is constant on \(\mathcal E\). Since \(f\) has an unconstrained global minimizer, this constant is zero, so \(f\equiv f^\star\), and the claims are trivial. Hence assume \(L>0\).

For fixed \(k\), the map
\[
    R\mapsto z_k(R)=x_0+(1-\beta)(x_k-x_0)+\beta R s_k
\]
is continuous. Therefore all terms in \(\mathcal Q_k(R)\) are continuous on \([0,\infty)\). Since \(\lambda>0\), \(L>0\), and \(\beta>0\), the term
\[
    \frac{\lambda L\beta^2}{2}(R-d_{k+1})^2
\]
is coercive as \(R\to+\infty\). Hence \(\mathcal Q_k(R)\to+\infty\) as \(R\to+\infty\), and \(\mathcal Q_k\) attains its minimum on the closed set \([0,\infty)\).

Convexity follows because
\[
    R\mapsto \|R s_k-y_k\|^2,
    \qquad
    R\mapsto \|(1-\beta)y_k+\beta R s_k\|^2
\]
are squared norms of affine functions of \(R\), hence convex. The remaining nonconstant terms are linear or quadratic in \(R\). Thus \(R\mapsto\mathcal Q_k(R)\) is convex.
\end{proof}

\subsection{One-step majorized inequality}

Define the Lyapunov function
\[
    \mathcal L_k
    :=
    \Delta_k
    +
    M\beta L\|y_k\|^2 .
\]
The quadratic term is created by the regularized scalar search and is used to control the momentum error; it is not a boundedness assumption.

Also define
\[
    e_k:=\nabla f(x_k)-m_{k+1}.
\]

\begin{lemma}[One-step majorized inequality]
\label{lem:df_muon_onestep}
Let Assumptions~\ref{assum:smoothness} and~\ref{assum:star_convex} hold, and let Algorithm~\ref{alg:pf_tr_muon} be run with \(0\le d_0\le D\). If \(D>0\) and
\[
    M\ge 2(1+2\rho),
\]
then, for every \(k\ge0\),
\[
\begin{aligned}
    \mathcal L_{k+1}
    +
    \rho L\|x_{k+1}-x_k\|^2
    \le
    \left(1-\frac{\beta}{2}\right)\mathcal L_k
    +
    C_0L\beta^2D^2
    +
    2\beta D\|e_k\|_*,
\end{aligned}
\]
where
\[
    C_0:=1+2\rho+\frac{\lambda}{2}+M.
\]
\end{lemma}

\begin{proof}
By \(L\)-smoothness and the definition of \(\mathcal Q_k\), for every \(R\ge0\),
\[
    f(z_k(R))
    \le
    f(x_k)
    +
    \beta\langle g_k,R s_k-y_k\rangle
    +
    \frac{L\beta^2}{2}\|R s_k-y_k\|^2 .
\]
At the chosen radius \(R_k\), this implies
\[
\begin{aligned}
    f(x_{k+1})
    &+
    M\beta L\|y_{k+1}\|^2
    +
    \rho L\|x_{k+1}-x_k\|^2 \\
    &\le
    \mathcal Q_k(R_k)
    -
    \frac{\lambda L\beta^2}{2}(R_k-d_{k+1})^2 .
\end{aligned}
\]
Since the last term on the right-hand side is nonpositive,
\[
    f(x_{k+1})
    +
    M\beta L\|y_{k+1}\|^2
    +
    \rho L\|x_{k+1}-x_k\|^2
    \le
    \mathcal Q_k(R_k).
\]
By optimality of \(R_k\) in the scalar search and since \(R=D\) is feasible,
\[
    \mathcal Q_k(R_k)\le \mathcal Q_k(D).
\]
Therefore
\[
\begin{aligned}
    f(x_{k+1})
    &+
    M\beta L\|y_{k+1}\|^2
    +
    \rho L\|x_{k+1}-x_k\|^2 \\
    &\le
    f(x_k)
    +
    \beta\langle g_k,D s_k-y_k\rangle
    +
    \frac{L\beta^2}{2}\|D s_k-y_k\|^2 \\
    &\quad
    +
    M\beta L\|(1-\beta)y_k+\beta D s_k\|^2
    +
    \rho L\beta^2\|D s_k-y_k\|^2 \\
    &\quad
    +
    \frac{\lambda L\beta^2}{2}(D-d_{k+1})^2 .
\end{aligned}
\]
By Lemma~\ref{lem:df_muon_dcert}, \(0\le d_{k+1}\le D\), hence
\[
    (D-d_{k+1})^2\le D^2.
\]

We now control the linear term. Decompose
\[
    D s_k-y_k=(u^\star-y_k)+(D s_k-u^\star).
\]
By Assumption~\ref{assum:star_convex},
\[
    \langle g_k,u^\star-y_k\rangle
    =
    \langle \nabla f(x_k),x^\star-x_k\rangle
    \le
    -\Delta_k.
\]
Write \(g_k=m_{k+1}+e_k\). Since
\[
    s_k\in\operatorname*{arg\,min}_{\|s\|\le1}
    \langle m_{k+1},s\rangle
\]
and \(\|u^\star/D\|\le1\), we have
\[
    \langle m_{k+1},D s_k\rangle
    \le
    \langle m_{k+1},u^\star\rangle .
\]
Therefore
\[
    \langle m_{k+1},D s_k-u^\star\rangle\le0.
\]
Moreover,
\[
    \|D s_k-u^\star\|\le 2D.
\]
Consequently,
\[
\begin{aligned}
    \beta\langle g_k,D s_k-y_k\rangle
    &\le
    -\beta\Delta_k
    +
    \beta\langle e_k,D s_k-u^\star\rangle \\
    &\le
    -\beta\Delta_k
    +
    2\beta D\|e_k\|_* .
\end{aligned}
\]

Since \(\|s_k\|\le1\),
\[
    \|D s_k-y_k\|^2
    \le
    2D^2+2\|y_k\|^2.
\]
Also, by convexity of \(v\mapsto \|v\|^2\),
\[
    \|(1-\beta)y_k+\beta D s_k\|^2
    \le
    (1-\beta)\|y_k\|^2+\beta D^2.
\]
Substituting the preceding inequalities gives
\[
\begin{aligned}
    \mathcal L_{k+1}
    &+
    \rho L\|x_{k+1}-x_k\|^2 \\
    &\le
    (1-\beta)\Delta_k
    +
    M\beta L(1-\beta)\|y_k\|^2 \\
    &\quad
    +
    L\beta^2(1+2\rho)\|y_k\|^2 \\
    &\quad
    +
    L\beta^2
    \left(1+2\rho+\frac{\lambda}{2}\right)D^2
    +
    M L\beta^2D^2
    +
    2\beta D\|e_k\|_* .
\end{aligned}
\]
Because \(M\ge2(1+2\rho)\),
\[
    L\beta^2(1+2\rho)\|y_k\|^2
    \le
    \frac{\beta}{2}M\beta L\|y_k\|^2.
\]
Also,
\[
    (1-\beta)\Delta_k
    \le
    \left(1-\frac{\beta}{2}\right)\Delta_k.
\]
Therefore
\[
\begin{aligned}
    \mathcal L_{k+1}
    &+
    \rho L\|x_{k+1}-x_k\|^2 \\
    &\le
    \left(1-\frac{\beta}{2}\right)
    \left(
        \Delta_k+M\beta L\|y_k\|^2
    \right)
    +
    C_0L\beta^2D^2
    +
    2\beta D\|e_k\|_*,
\end{aligned}
\]
which is the desired inequality.
\end{proof}

\subsection{Momentum-error control}

Let
\[
    a_i:=\|x_i-x_{i-1}\|,
    \qquad
    r:=1-\frac{\beta}{2},
    \qquad
    q:=1-\alpha.
\]

\begin{lemma}[Weighted momentum-error bound]
\label{lem:df_muon_momentum}
Suppose Assumption~\ref{assum:smoothness} holds and
\[
    \alpha>\frac{\beta}{2}.
\]
Then, for every \(T\ge1\),
\[
    \sum_{k=0}^{T-1}r^{T-1-k}\|e_k\|_*
    \le
    \frac{(1-\alpha)L}{\alpha-\beta/2}
    \sqrt{\frac{2}{\beta}}
    \left(
        \sum_{i=1}^T r^{T-i}a_i^2
    \right)^{1/2}.
\]
\end{lemma}

\begin{proof}
Since \(m_0=\nabla f(x_0)\), we have \(e_0=0\). For \(k\ge1\),
\[
\begin{aligned}
    e_k
    &=
    \nabla f(x_k)-m_{k+1}  \\
    &=
    (1-\alpha)
    \bigl(\nabla f(x_k)-m_k\bigr) \\
    &=
    q\bigl(\nabla f(x_k)-\nabla f(x_{k-1})+e_{k-1}\bigr).
\end{aligned}
\]
By \(L\)-smoothness,
\[
    \|e_k\|_*
    \le
    qL a_k+q\|e_{k-1}\|_*.
\]
Unrolling this recursion gives
\[
    \|e_k\|_*
    \le
    qL\sum_{i=1}^k q^{k-i}a_i.
\]
Therefore
\[
\begin{aligned}
    \sum_{k=0}^{T-1}r^{T-1-k}\|e_k\|_*
    &\le
    qL
    \sum_{k=1}^{T-1}r^{T-1-k}
    \sum_{i=1}^k q^{k-i}a_i  \\
    &=
    qL
    \sum_{i=1}^{T-1}a_i
    \sum_{k=i}^{T-1}r^{T-1-k}q^{k-i}.
\end{aligned}
\]
Since \(q<r\), or equivalently \(\alpha>\beta/2\),
\[
    \sum_{k=i}^{T-1}r^{T-1-k}q^{k-i}
    \le
    \frac{r^{T-i}}{r-q}.
\]
Hence
\[
    \sum_{k=0}^{T-1}r^{T-1-k}\|e_k\|_*
    \le
    \frac{qL}{r-q}
    \sum_{i=1}^{T-1}r^{T-i}a_i.
\]
By Cauchy--Schwarz,
\[
    \sum_{i=1}^{T-1}r^{T-i}a_i
    \le
    \left(\sum_{i=1}^{T-1}r^{T-i}\right)^{1/2}
    \left(\sum_{i=1}^{T-1}r^{T-i}a_i^2\right)^{1/2}.
\]
Since
\[
    \sum_{i=1}^{T-1}r^{T-i}
    \le
    \frac{1}{1-r}
    =
    \frac{2}{\beta},
\]
and extending the second sum to \(i=T\) only increases it, we obtain
\[
    \sum_{k=0}^{T-1}r^{T-1-k}\|e_k\|_*
    \le
    \frac{qL}{r-q}
    \sqrt{\frac{2}{\beta}}
    \left(
        \sum_{i=1}^T r^{T-i}a_i^2
    \right)^{1/2}.
\]
Substituting \(q=1-\alpha\) and \(r-q=\alpha-\beta/2\) proves the claim.
\end{proof}

\subsection{Proof of Theorem~\ref{thm:df_muon}}

\begin{proof}[Proof of Theorem~\ref{thm:df_muon}]
If \(D=0\), then \(x_0=x^\star\). Since \(x^\star\) is an unconstrained global minimizer, \(\nabla f(x^\star)=0\). The warm start gives \(m_0=0\), and the majorized scalar search is minimized at \(R_0=0\); hence \(x_1=x^\star\). Inductively, \(x_k=x^\star\) for all \(k\), so the claim is immediate. Assume \(D>0\).

Let
\[
    A_T:=\sum_{i=1}^T r^{T-i}a_i^2,
    \qquad
    r:=1-\frac{\beta}{2}.
\]
Multiplying Lemma~\ref{lem:df_muon_onestep} by \(r^{T-1-k}\) and summing from \(k=0\) to \(T-1\), we get
\[
\begin{aligned}
    \mathcal L_T+\rho L A_T
    &\le
    r^T\mathcal L_0
    +
    C_0L\beta^2D^2
    \sum_{k=0}^{T-1}r^{T-1-k} \\
    &\quad
    +
    2\beta D
    \sum_{k=0}^{T-1}r^{T-1-k}\|e_k\|_* .
\end{aligned}
\]
Since \(1-r=\beta/2\),
\[
    \sum_{k=0}^{T-1}r^{T-1-k}
    \le
    \frac{2}{\beta}.
\]
Therefore
\[
    \mathcal L_T +\rho L A_T
    \le
    r^T\mathcal L_0
    +
    2C_0L\beta D^2
    +
    2\beta D
    \sum_{k=0}^{T-1}r^{T-1-k}\|e_k\|_* .
\]
By Lemma~\ref{lem:df_muon_momentum},
\[
    \sum_{k=0}^{T-1}r^{T-1-k}\|e_k\|_*
    \le
    \frac{(1-\alpha)L}{\alpha-\beta/2}
    \sqrt{\frac{2}{\beta}}
    A_T^{1/2}.
\]
Thus
\[
\begin{aligned}
    2\beta D
    \sum_{k=0}^{T-1}r^{T-1-k}\|e_k\|_*
    &\le
    \frac{2\sqrt2(1-\alpha)LD\sqrt{\beta}}
         {\alpha-\beta/2}
    A_T^{1/2}.
\end{aligned}
\]
Using Young's inequality,
\[
    ab\le \frac{\rho L}{2}a^2+\frac{b^2}{2\rho L},
\]
with
\[
    a=A_T^{1/2},
    \qquad
    b=
    \frac{2\sqrt2(1-\alpha)LD\sqrt{\beta}}
         {\alpha-\beta/2},
\]
we obtain
\[
\begin{aligned}
    2\beta D
    \sum_{k=0}^{T-1}r^{T-1-k}\|e_k\|_*
    &\le
    \frac{\rho L}{2}A_T
    +
    \frac{4(1-\alpha)^2L\beta D^2}
         {\rho(\alpha-\beta/2)^2}.
\end{aligned}
\]
Consequently,
\[
    \mathcal L_T +\frac{\rho L}{2}A_T
    \le
    r^T \mathcal L_0
    +
    2C_0L\beta D^2
    +
    \frac{4(1-\alpha)^2L\beta D^2}
         {\rho(\alpha-\beta/2)^2}.
\]
Dropping the nonnegative term \((\rho L/2)A_T\), and using
\[
    \mathcal L_T\ge f(x_T)-f^\star,
    \qquad
    \mathcal L_0=f(x_0)-f^\star,
\]
gives
\[
\begin{aligned}
    f(x_T)-f^\star
    \le
    r^T(f(x_0)-f^\star)
    +
    \left[
        2C_0
        +
        \frac{4(1-\alpha)^2}
        {\rho(\alpha-\beta/2)^2}
    \right]
    L\beta D^2.
\end{aligned}
\]
Substituting \(r=1-\beta/2\) proves the first claim.

For the explicit rate, choose
\[
    \beta=\min\left\{\alpha,\frac{2\log(T+1)}{T}\right\}.
\]
Whenever \(2\log(T+1)/T\le\alpha\), we have \(\beta=2\log(T+1)/T\), and hence
\[
    \left(1-\frac{\beta}{2}\right)^T
    \le
    \exp\left(-\frac{\beta T}{2}\right)
    =
    \frac{1}{T+1}.
\]
Moreover, since \(\beta\le\alpha\),
\[
    \alpha-\frac{\beta}{2}\ge\frac{\alpha}{2}.
\]
Therefore
\[
\begin{aligned}
    f(x_T)-f^\star
    &\le
    \frac{f(x_0)-f^\star}{T+1} \\
    &\quad+
    2\left[
        2C_0
        +
        \frac{16(1-\alpha)^2}{\rho\alpha^2}
    \right]
    \frac{LD^2\log(T+1)}{T}.
\end{aligned}
\]

It remains only to bound the initial gap by the same \(LD^2\) scale. Since \(x^\star\) is an unconstrained global minimizer and \(f\) is differentiable, \(\nabla f(x^\star)=0\). By star-convexity, duality, and smoothness,
\[
\begin{aligned}
    f(x_0)-f(x^\star)
    &\le
    \langle \nabla f(x_0),x_0-x^\star\rangle
    \le
    \|\nabla f(x_0)\|_*D \\
    &=
    \|\nabla f(x_0)-\nabla f(x^\star)\|_*D
    \le
    LD^2 .
\end{aligned}
\]
Thus
\[
    f(x_T)-f^\star
    =
    \widetilde O\!\left(\frac{LD^2}{T}\right),
    \qquad
    D=\|x_0-x^\star\|,
\]
with constants depending only on \(\alpha,\rho,\lambda\), and \(M\), but not on \(D\).
\end{proof}

\begin{remark}[What is distance-free?]
The algorithm never receives \(D=\|x_0-x^\star\|\) as an input. The proof uses \(D\) only as a hidden comparator radius in the scalar search. The scalar proxy \(d_k\) is a D-adaptation-style lower certificate and is used only through the valid inequality \(d_k\le D\). Thus the result avoids both the fixed-radius choice \(\eta=\beta D\) and the need to prove that \(d_k\) discovers \(D\).
\end{remark}

\begin{remark}[Majorized versus objective-value scalar search]
The scalar search minimizes a regularized smoothness majorant rather than the true one-dimensional loss \(R\mapsto f(z_k(R))\). This avoids a global search over the objective value along the ray and yields a convex one-dimensional subproblem. The proof still compares the chosen radius to the hidden reference radius \(R=D\), which is why the distance parameter appears only in the analysis.
\end{remark}
\section{Additional Experiments}
\label{app:additional_experiments}

\subsection{Implementation Details for DF-Muon}
\label{app:df_muon_implementation}

The deterministic DF-Muon analysis in Section~\ref{sec:df_muon} is stated for the exact minimizer of a norm-generic one-dimensional majorized model. In the neural-network experiments, we use a practical stochastic implementation of the same scalar-calibration principle. The Muon direction itself is unchanged: for matrix-valued parameters we use the standard orthogonalized Muon direction associated with the spectral--nuclear geometry. The scalar radius, however, is computed using product-Frobenius aggregate statistics over the Muon-updated matrix blocks, rather than by solving the exact spectral-geometry scalar problem.

Concretely, let \(u_{t,j}\) denote the Muon direction for matrix block \(j\), let \(g_{t,j}\) be the corresponding stochastic gradient, and let \(y_{t,j}=x_{t,j}-x_{0,j}\). At each optimizer step we aggregate
\[
    A_t=\sum_j \|u_{t,j}\|_F^2,\qquad
    B_t=\sum_j \langle y_{t,j},u_{t,j}\rangle,\qquad
    C_t=\sum_j \|y_{t,j}\|_F^2,\qquad
    G_t=\sum_j \langle g_{t,j},u_{t,j}\rangle .
\]
These quantities define a cheap one-dimensional Frobenius-proxy majorized score for the scalar radius. We minimize this score over the capped interval \([\eta_{\min},\eta_{\max}]\) by a small grid search followed by local refinement, smooth the selected scale, and then multiply it by the common warmup--cosine schedule before applying the update. This procedure uses no additional forward or backward passes; it only reuses quantities already available during the optimizer step.

Thus, the experiments should be viewed as a practical Frobenius-proxy implementation of the majorized DF-Muon radius-selection principle, rather than as an exact solver for the norm-generic deterministic scalar problem in Theorem~\ref{thm:df_muon}. This distinction is important in spectral--nuclear Muon geometry: the exact scalar problem remains one-dimensional and convex, but may be nonsmooth if written directly in the spectral norm. The Frobenius-proxy implementation avoids this cost while retaining the Muon orthogonalized direction.

Unless stated otherwise, the DF-Muon experiments use a grid of \(21\) candidate scales and \(6\) local refinement steps. For GPT-124M/WikiText-103, we use \(\eta_{\min}=0.006\), initial scale \(0.015\), cap \(\eta_{\max}=0.03\), and smoothing coefficient \(0.70\). For the regularized ViT-Tiny/CIFAR-100 experiment, we use the same implementation with \(\eta_{\max}=0.01\), selected from the cap diagnostic in Appendix~\ref{app:vit_cifar100_regularized_diagnostics}.

\subsection{Additional GPT-124M/WikiText-103 Diagnostics}
\label{app:gpt124m_wikitext103_diagnostics}

This section provides additional diagnostics for the GPT-124M/WikiText-103 experiment in Section~\ref{subsec:gpt124m_wikitext103}. We include the fixed-Muon learning-rate sweep used to select the tuned baseline, the DA-Muon cap sweep used to select its cap, a one-seed DF-Muon majorized-model diagnostic used to select the no-center variant, and the observed base scales selected by the adaptive variants.

\paragraph{Fixed-Muon learning-rate sweep.}
Table~\ref{tab:gpt124m_muon_lr_sweep} reports the fixed-Muon learning-rate sweep on WikiText-103 for one representative seed. The sweep selects $\eta=0.005$ as the best fixed-Muon learning rate. Larger fixed learning rates are substantially worse in this setting, indicating that the large Transformer benchmark is sensitive to the scalar scale.

\begin{table}[h]
\centering
\caption{
Fixed-Muon learning-rate sweep on the GPT-style Transformer/WikiText-103
benchmark for one representative seed. Runtime is reported relative to
Muon with $\eta=0.015$.
}
\label{tab:gpt124m_muon_lr_sweep}
\footnotesize
\setlength{\tabcolsep}{4pt}
\begin{tabular}{@{}lcccc@{}}
\toprule
\textbf{Method}
& \textbf{Train loss}
& \textbf{Val. loss}
& \textbf{Mean $\eta$}
& \textbf{Rel. time} \\
\midrule
Muon, $\eta=0.005$
& $3.7558$
& $\mathbf{3.7650}$
& $0.0050$
& $1.01\times$ \\
Muon, $\eta=0.01$
& $3.7755$
& $3.7834$
& $0.0100$
& $1.00\times$ \\
Muon, $\eta=0.015$
& $3.8081$
& $3.8130$
& $0.0150$
& $1.00\times$ \\
Muon, $\eta=0.02$
& $4.0177$
& $4.0166$
& $0.0200$
& $1.00\times$ \\
Muon, $\eta=0.03$
& $4.8027$
& $4.7822$
& $0.0300$
& $1.00\times$ \\
\bottomrule
\end{tabular}
\end{table}

\begin{figure}[h]
    \centering
    \begin{minipage}{0.48\linewidth}
        \centering
        \includegraphics[width=\linewidth]{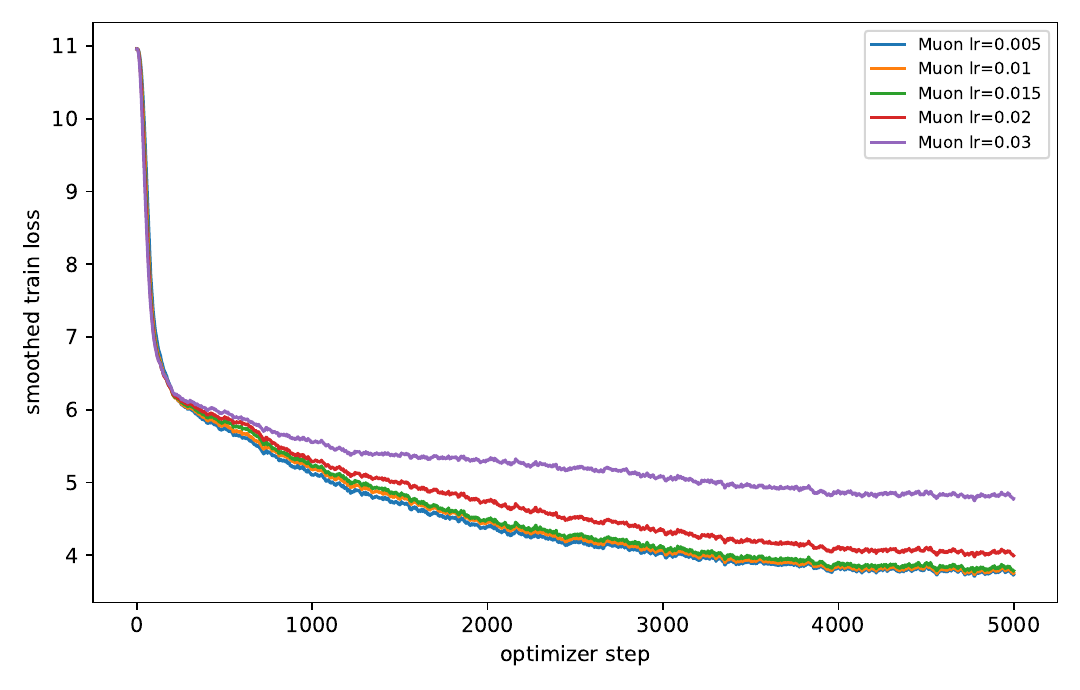}
        \centerline{\small (a) Training loss}
    \end{minipage}
    \hfill
    \begin{minipage}{0.48\linewidth}
        \centering
        \includegraphics[width=\linewidth]{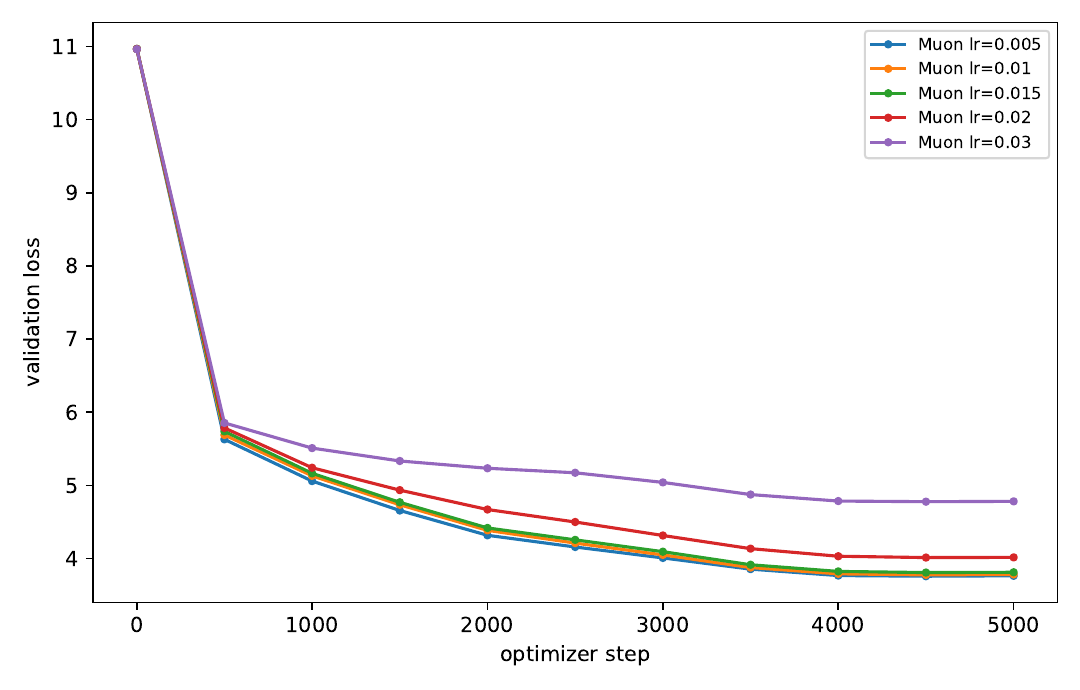}
        \centerline{\small (b) Validation loss}
    \end{minipage}
    \caption{
    Fixed-Muon learning-rate sweep on GPT-124M/WikiText-103. The smaller
    learning rate $\eta=0.005$ gives the best validation loss among the
    tested fixed scales.
    }
    \label{fig:gpt124m_muon_lr_sweep}
\end{figure}

\paragraph{DA-Muon cap sweep.}
The default DA-Muon cap $\eta_{\max}=0.03$ is too aggressive on this benchmark. Table~\ref{tab:gpt124m_da_cap_sweep} reports a cap sweep. Tuning the cap makes DA-Muon competitive, with the best result at $\eta_{\max}=0.01$, but it remains slightly worse than the tuned fixed-Muon baseline and DF-Muon.

\begin{table}[h]
\centering
\caption{
DA-Muon cap sweep on GPT-124M/WikiText-103 for one representative seed.
Runtime is reported relative to the tuned fixed-Muon baseline with
$\eta=0.005$.
}
\label{tab:gpt124m_da_cap_sweep}
\footnotesize
\setlength{\tabcolsep}{4pt}
\begin{tabular}{@{}lcccc@{}}
\toprule
\textbf{Method}
& \textbf{Train loss}
& \textbf{Val. loss}
& \textbf{Mean $\eta$}
& \textbf{Rel. time} \\
\midrule
AdamW
& $3.9941$
& $3.9874$
& $0.0010$
& $0.66\times$ \\
Best fixed Muon, $\eta=0.005$
& $3.7550$
& $\mathbf{3.7636}$
& $0.0050$
& $1.00\times$ \\
DA-Muon, $\eta_{\max}=0.005$
& $3.7728$
& $3.7818$
& $0.0050$
& $1.03\times$ \\
DA-Muon, $\eta_{\max}=0.01$
& $3.7680$
& $3.7779$
& $0.0100$
& $1.03\times$ \\
DA-Muon, $\eta_{\max}=0.015$
& $3.7724$
& $3.7793$
& $0.0150$
& $1.03\times$ \\
DA-Muon, $\eta_{\max}=0.02$
& $3.8854$
& $3.8908$
& $0.0200$
& $1.03\times$ \\
DA-Muon, $\eta_{\max}=0.03$
& $4.4655$
& $4.4528$
& $0.0300$
& $1.02\times$ \\
\bottomrule
\end{tabular}
\end{table}

\begin{figure}[h]
    \centering
    \begin{minipage}{0.48\linewidth}
        \centering
        \includegraphics[width=\linewidth]{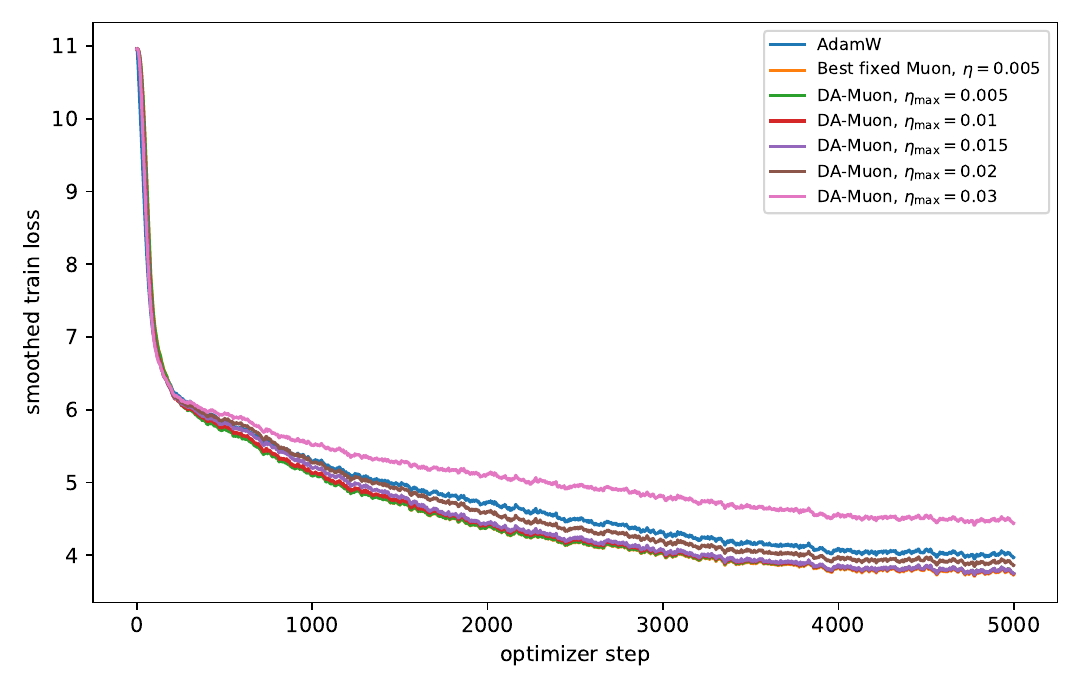}
        \centerline{\small (a) Training loss}
    \end{minipage}
    \hfill
    \begin{minipage}{0.48\linewidth}
        \centering
        \includegraphics[width=\linewidth]{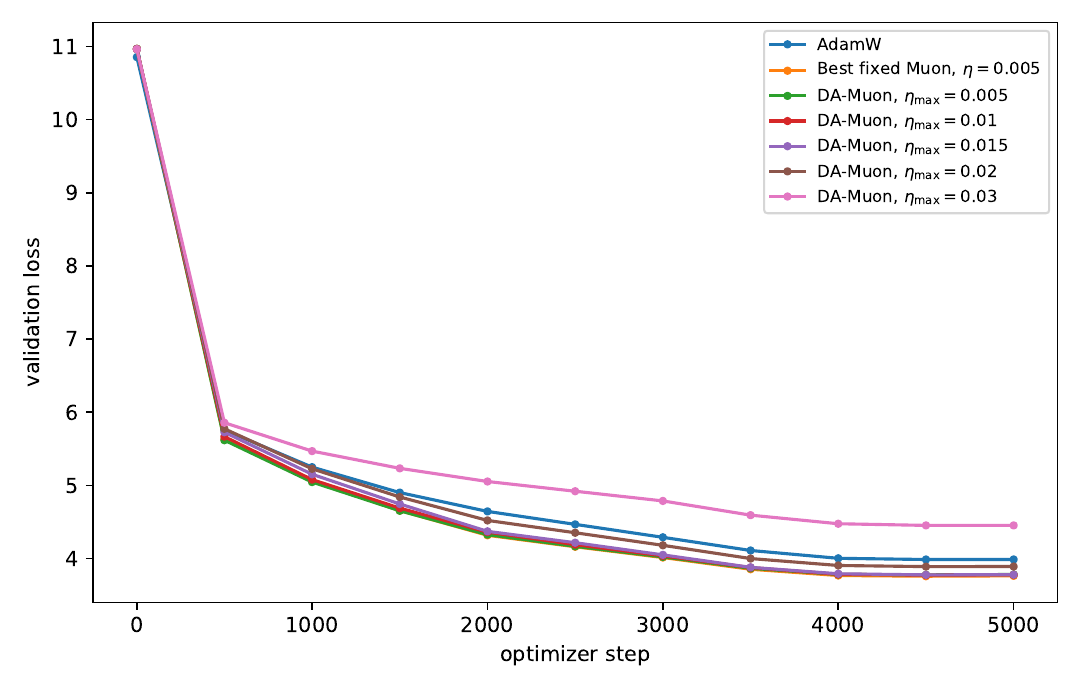}
        \centerline{\small (b) Validation loss}
    \end{minipage}
    \caption{
    DA-Muon cap sweep on GPT-124M/WikiText-103. Smaller caps substantially
    improve DA-Muon over the default $\eta_{\max}=0.03$ setting.
    }
    \label{fig:gpt124m_da_cap_sweep}
\end{figure}

\paragraph{DF-Muon majorized-model diagnostic.}
For DF-Muon, we ran a one-seed diagnostic over five nearby variants of the practical majorized scalar model. All variants use the same cap, floor, smoothing coefficient, grid size, and number of local refinement steps. The no-center variant, which sets the center-penalty coefficient to zero while keeping the step and proxy coefficients fixed, gives the best validation loss. We therefore use this variant in the final three-seed GPT-124M/WikiText-103 comparison.

\begin{table}[h]
\centering
\caption{
One-seed DF-Muon majorized-model diagnostic on GPT-124M/WikiText-103. All
variants use \(\eta_{\min}=0.006\), initial scale \(0.015\),
\(\eta_{\max}=0.03\), smoothing coefficient \(0.70\), \(21\) grid points, and
\(6\) local refinement steps.
}
\label{tab:app_gpt124m_df_major_diagnostic}
\footnotesize
\setlength{\tabcolsep}{4pt}
\begin{tabular}{@{}lcccc@{}}
\toprule
\textbf{Variant}
& \textbf{Step coeff.}
& \textbf{Center coeff.}
& \textbf{Proxy coeff.}
& \textbf{Val. loss} \\
\midrule
Default
& \(0.10\)
& \(0.02\)
& \(0.10\)
& \(3.6947\) \\
Less conservative
& \(0.05\)
& \(0.02\)
& \(0.10\)
& \(3.6906\) \\
More conservative
& \(0.15\)
& \(0.02\)
& \(0.10\)
& \(3.6960\) \\
No-center
& \(0.10\)
& \(0.00\)
& \(0.10\)
& \(\mathbf{3.6677}\) \\
Less proxy
& \(0.10\)
& \(0.02\)
& \(0.05\)
& \(3.6969\) \\
\bottomrule
\end{tabular}
\end{table}

\paragraph{Observed adaptive scales.}
Because all methods use a warmup--cosine schedule, the schedule-multiplied actual step size approaches zero near the end of training. We therefore report the base scale $\eta_t$ selected by each adaptive rule before multiplying by the common schedule. Table~\ref{tab:gpt124m_wikitext103_observed_eta} reports the mean base scale over the final $20\%$ of training, averaged over seeds.

\begin{table}[h]
\centering
\caption{
Observed base scales on GPT-124M/WikiText-103 over three seeds. For adaptive
methods, the mean base scale is averaged over the final $20\%$ of training and
then averaged over seeds.
}
\label{tab:gpt124m_wikitext103_observed_eta}
\footnotesize
\setlength{\tabcolsep}{5pt}
\begin{tabular}{@{}lcc@{}}
\toprule
\textbf{Method}
& \textbf{Selection rule}
& \textbf{Mean base $\eta$} \\
\midrule
AdamW
& fixed
& $0.0010$ \\
Best fixed Muon
& fixed sweep
& $0.0050$ \\
DF-Muon
& adaptive, no-center
& $0.0300$ \\
DA-Muon
& adaptive cap sweep
& $0.0100$ \\
SC-Muon
& adaptive
& $0.0263$ \\
\bottomrule
\end{tabular}
\end{table}

\begin{figure}[h]
    \centering
    \includegraphics[width=0.82\linewidth]{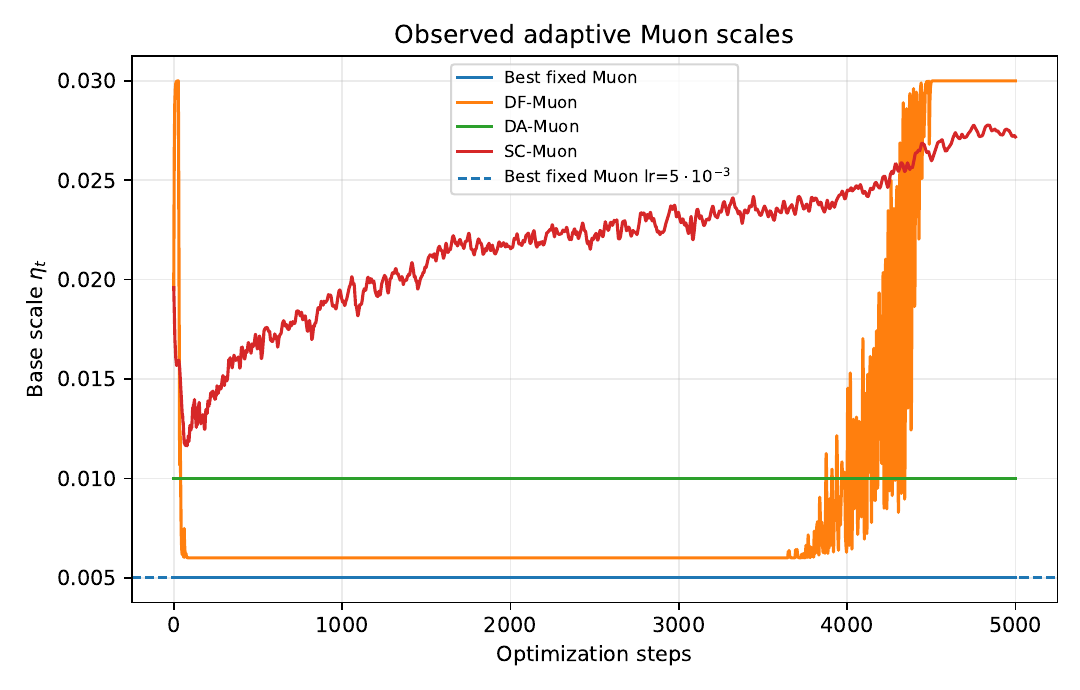}
    \caption{
    Observed base scales selected by the Muon variants on GPT-124M/WikiText-103
    in a representative seed. The tuned fixed-Muon baseline uses
    $\eta=0.005$, while DA-Muon uses the tuned cap $\eta_{\max}=0.01$.
    DF-Muon follows a data-dependent scale trajectory and achieves the best
    three-seed validation loss in Table~\ref{tab:gpt124m_wikitext103_three_seeds}.
    }
    \label{fig:gpt124m_wikitext103_observed_eta}
\end{figure}

\paragraph{Interpretation.}
The diagnostics show that the improvement of DF-Muon is not due to comparing against an under-tuned fixed-Muon baseline. The fixed-Muon sweep selects $\eta=0.005$, and the main three-seed comparison uses this tuned baseline. The DA-Muon sweep shows that the radius-based rule is sensitive to its cap in large Transformer training: the default $\eta_{\max}=0.03$ is too aggressive, whereas $\eta_{\max}=0.01$ is competitive. The DF-Muon majorized-model diagnostic selects the no-center variant, which removes an overly conservative center penalty in this GPT-124M setting. With this variant, DF-Muon remains the strongest method in the three-seed comparison, improving over the tuned fixed-Muon baseline with only a small runtime overhead.

\subsection{NanoGPT/WikiText-2 Diagnostics}
\label{app:nanogpt_wikitext2_diagnostics}

We include a smaller Transformer language-modeling experiment to complement the GPT-124M/WikiText-103 results in Section~\ref{subsec:gpt124m_wikitext103}. The goal is not to establish a strong language-modeling benchmark, but to check whether the adaptive Muon scaling rules show the same qualitative behavior in a short-budget stochastic Transformer setting.

\paragraph{Experimental setup.}
We train a NanoGPT-style decoder-only Transformer on WikiText-2 using GPT-2 tokenization. The model has $6$ layers, $6$ attention heads, embedding dimension $384$, context length $128$, and approximately $30$M parameters. All methods are run for $T=1000$ optimizer steps. AdamW uses learning rate $10^{-3}$. The fixed-Muon baseline uses $\eta=0.015$, selected from a small fixed-Muon learning-rate grid. The adaptive Muon variants use the same maximum base scale $\eta_{\max}=0.03$. As in the main experiment, Muon is applied to matrix-valued parameters, while biases, normalization parameters, embeddings, and other non-matrix parameters are optimized with AdamW.

\paragraph{Three-seed comparison.}
Table~\ref{tab:app_nanogpt_wikitext2_three_seeds} reports the aggregate over three seeds. The updated majorized DF-Muon implementation gives the best mean training and validation loss. DA-Muon and SC-Muon also improve over the tuned fixed-Muon baseline on validation loss, but the gains are smaller.

\begin{table}[h]
\centering
\caption{
NanoGPT/WikiText-2 results over three seeds for the updated majorized DF-Muon
implementation. We report mean $\pm$ standard deviation over seeds
$\{42,1337,2024\}$. The fixed-Muon baseline uses $\eta=0.015$; adaptive Muon
variants use $\eta_{\max}=0.03$.
}
\label{tab:app_nanogpt_wikitext2_three_seeds}
\footnotesize
\setlength{\tabcolsep}{4pt}
\begin{tabular}{@{}lccc@{}}
\toprule
\textbf{Method}
& \textbf{Train loss}
& \textbf{Val. loss}
& \textbf{Mean base $\eta$} \\
\midrule
AdamW
& $5.5576 \pm 0.0163$
& $5.7926 \pm 0.0339$
& $0.0010$ \\
Best fixed Muon
& $5.4489 \pm 0.0196$
& $5.7162 \pm 0.0256$
& $0.0150$ \\
DF-Muon
& $\mathbf{5.3073 \pm 0.0156}$
& $\mathbf{5.6630 \pm 0.0297}$
& $0.0202$ \\
DA-Muon
& $5.4445 \pm 0.0205$
& $5.7048 \pm 0.0311$
& $0.0300$ \\
SC-Muon
& $5.4122 \pm 0.0176$
& $5.6955 \pm 0.0356$
& $0.0202$ \\
\bottomrule
\end{tabular}
\end{table}

\paragraph{Representative curves.}
Figure~\ref{fig:app_nanogpt_wikitext2_curves} shows representative training and validation curves. The validation curves confirm that the improvement is not restricted to the training objective: DF-Muon gives the lowest validation loss on the representative seed, while DA-Muon and SC-Muon remain competitive with the fixed-Muon baseline.

\begin{figure}[h]
    \centering
    \begin{minipage}{0.48\linewidth}
        \centering
        \includegraphics[width=\linewidth]{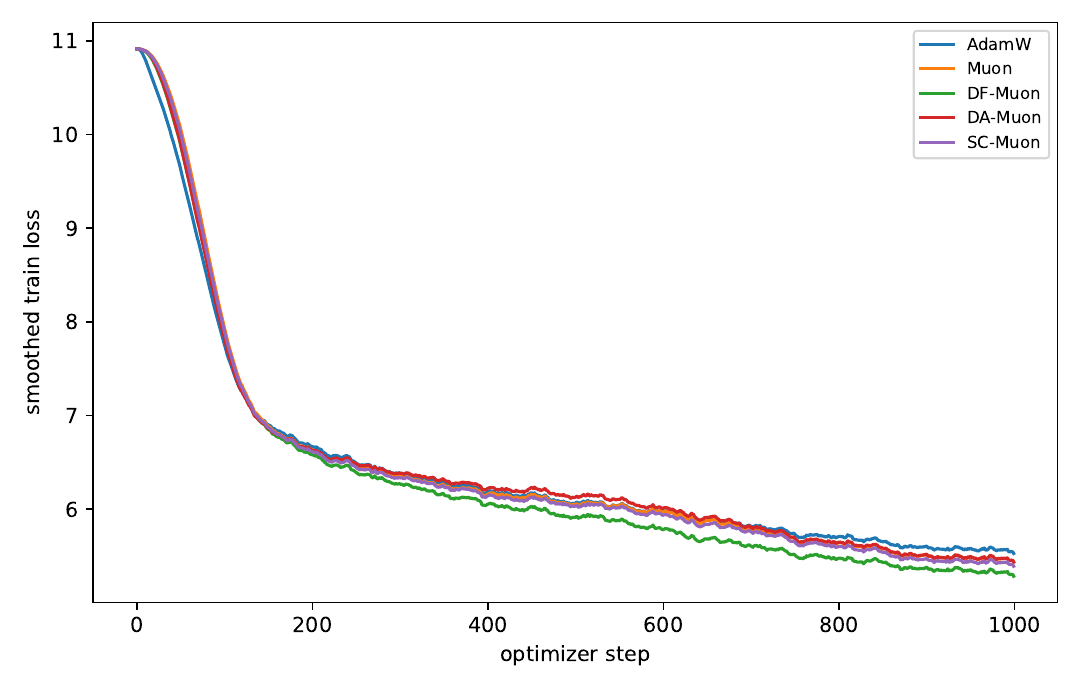}
        \centerline{\small (a) Training loss}
    \end{minipage}
    \hfill
    \begin{minipage}{0.48\linewidth}
        \centering
        \includegraphics[width=\linewidth]{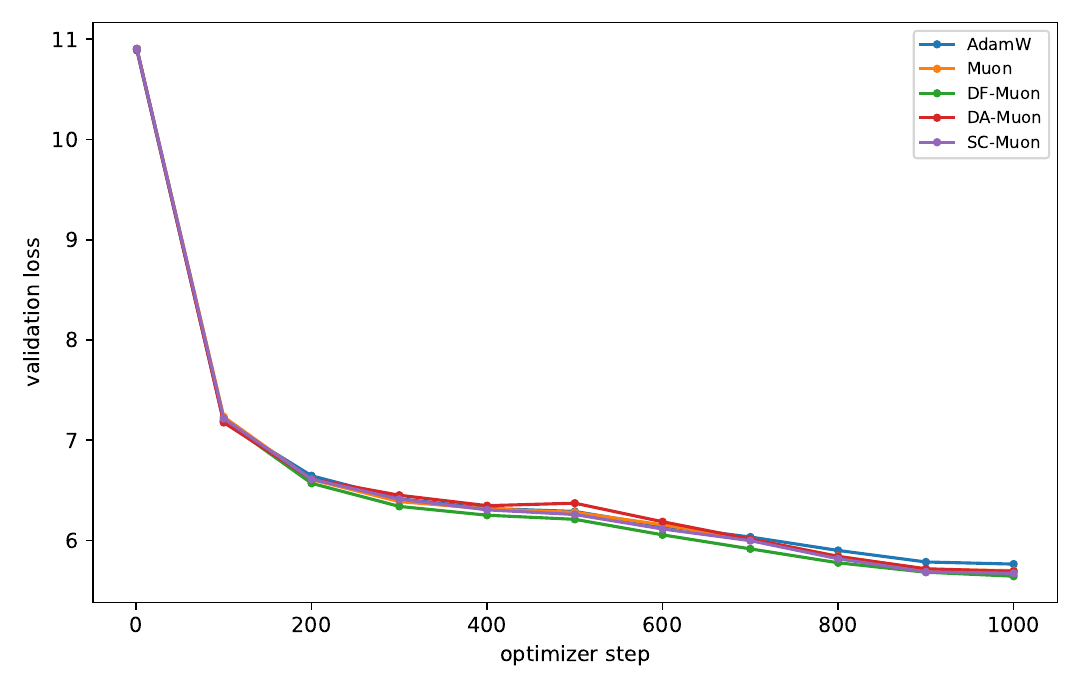}
        \centerline{\small (b) Validation loss}
    \end{minipage}
    \caption{
    Representative NanoGPT/WikiText-2 curves for seed $42$. We compare AdamW,
    the tuned fixed-Muon baseline, and the three adaptive Muon variants.
    Table~\ref{tab:app_nanogpt_wikitext2_three_seeds} reports the aggregate
    over three seeds.
    }
    \label{fig:app_nanogpt_wikitext2_curves}
\end{figure}

\paragraph{Observed adaptive scales.}
Because all methods use a warmup--cosine schedule, the schedule-multiplied actual step size approaches zero near the end of training. We therefore report the base scale $\eta_t$ selected by each adaptive rule before multiplication by the common schedule. Table~\ref{tab:app_nanogpt_wikitext2_observed_eta} reports the final-$20\%$ mean base scale averaged over seeds, and Figure~\ref{fig:app_nanogpt_wikitext2_observed_eta} shows the representative scale trajectories.

DA-Muon saturates the cap $\eta_{\max}=0.03$ on this small Transformer task. SC-Muon selects a moderate scale around $0.02$. DF-Muon follows a different trajectory: it first becomes conservative and then increases its scale late in training. This behavior is consistent with the validation results in Table~\ref{tab:app_nanogpt_wikitext2_three_seeds}, where DF-Muon gives the strongest aggregate performance.

\begin{table}[h]
\centering
\caption{
Observed base scales on NanoGPT/WikiText-2 over three seeds. For adaptive
methods, the mean base scale is averaged over the final $20\%$ of training and
then averaged over seeds.
}
\label{tab:app_nanogpt_wikitext2_observed_eta}
\footnotesize
\setlength{\tabcolsep}{5pt}
\begin{tabular}{@{}lcc@{}}
\toprule
\textbf{Method}
& \textbf{Selection rule}
& \textbf{Mean base $\eta$} \\
\midrule
AdamW
& fixed
& $0.0010$ \\
Best fixed Muon
& fixed
& $0.0150$ \\
DF-Muon
& adaptive
& $0.0202$ \\
DA-Muon
& adaptive
& $0.0300$ \\
SC-Muon
& adaptive
& $0.0202$ \\
\bottomrule
\end{tabular}
\end{table}

\begin{figure}[h]
    \centering
    \includegraphics[width=0.82\linewidth]{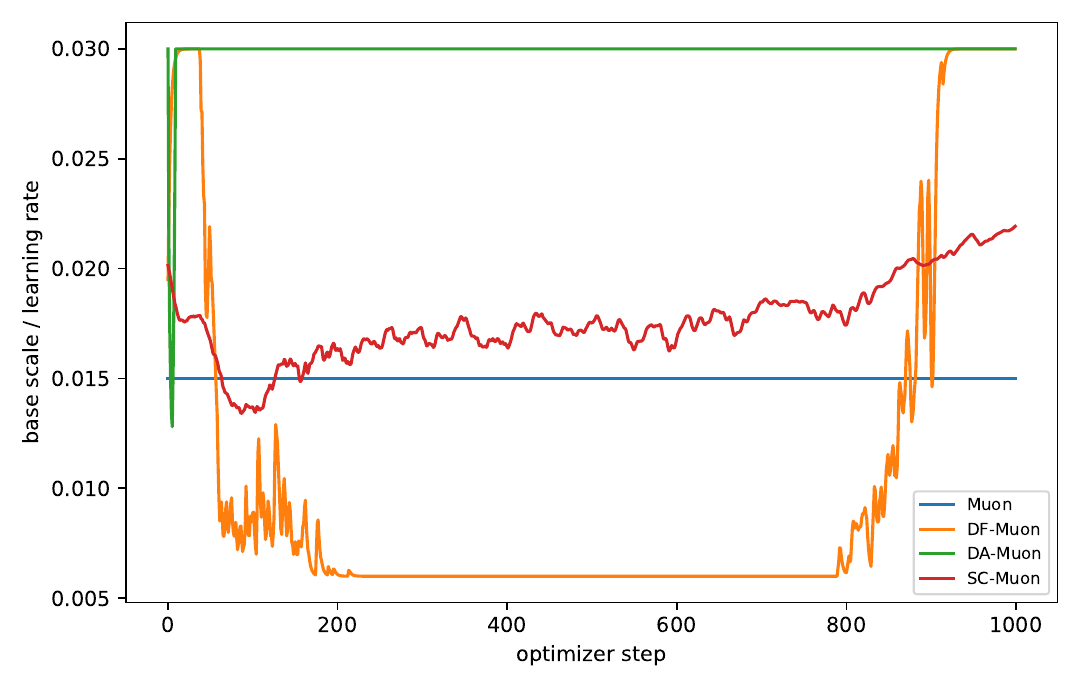}
    \caption{
    Observed base scales selected by the Muon variants on NanoGPT/WikiText-2
    in representative seed $42$. The fixed-Muon baseline uses $\eta=0.015$,
    DA-Muon saturates the cap $\eta_{\max}=0.03$, SC-Muon stabilizes around
    $0.02$, and DF-Muon follows a conservative-then-increasing scale
    trajectory.
    }
    \label{fig:app_nanogpt_wikitext2_observed_eta}
\end{figure}

\paragraph{Interpretation.}
The NanoGPT experiment supports the same qualitative conclusion as the larger GPT-124M/WikiText-103 benchmark: adaptive scalar scaling can improve Muon even against a tuned fixed-Muon baseline. In this smaller setting, all three adaptive variants improve mean validation loss over fixed Muon. The updated majorized DF-Muon is the strongest method, reducing mean validation loss from $5.7162$ for fixed Muon to $5.6630$. DA-Muon and SC-Muon also improve over fixed Muon, reaching validation losses $5.7048$ and $5.6955$, respectively. The observed-scale diagnostics show that these improvements arise from different scale-selection behavior rather than from a single shared effective learning rate.

\subsection{Regularized ViT-Tiny/CIFAR-100 Diagnostics}
\label{app:vit_cifar100_regularized_diagnostics}

This section provides additional diagnostics for the regularized ViT-Tiny/CIFAR-100 experiment in Section~\ref{subsec:vit_cifar100_regularized}. The goal is to document the fixed-Muon scale selection, the DF-Muon cap selection, and the effective scales used by the adaptive methods.

\paragraph{Fixed-Muon learning-rate diagnostic.}
We first ran a one-seed fixed-Muon learning-rate diagnostic on ViT-Tiny/CIFAR-100. Table~\ref{tab:app_vit_cifar100_muon_lr_sweep} reports the sweep. The smallest tested Muon scale, \(\eta=0.001\), gives the best validation cross-entropy and top-1 accuracy in this diagnostic, and is therefore used as the tuned fixed-Muon baseline in Table~\ref{tab:vit_cifar100_regularized}.

\begin{table}[h]
\centering
\caption{
Fixed-Muon learning-rate diagnostic on ViT-Tiny/CIFAR-100 for one
representative seed. CIFAR-100 images are resized to \(224\times224\).
Runtime is reported relative to the reference fixed-Muon run.
}
\label{tab:app_vit_cifar100_muon_lr_sweep}
\footnotesize
\setlength{\tabcolsep}{4pt}
\begin{tabular}{@{}lccccc@{}}
\toprule
\textbf{Method}
& \textbf{Train CE}
& \textbf{Val. CE}
& \textbf{Top-1}
& \textbf{Mean \(\eta\)}
& \textbf{Rel. time} \\
\midrule
Muon, \(\eta=0.001\)
& \(0.3589\)
& \(\mathbf{1.6752}\)
& \(\mathbf{57.74}\)
& \(0.0010\)
& \(1.01\times\) \\
Muon, \(\eta=0.002\)
& \(0.0775\)
& \(2.0090\)
& \(55.85\)
& \(0.0020\)
& \(0.99\times\) \\
Muon, \(\eta=0.003\)
& \(0.0327\)
& \(2.2269\)
& \(53.85\)
& \(0.0030\)
& \(0.99\times\) \\
Muon, \(\eta=0.005\)
& \(0.0295\)
& \(2.3036\)
& \(51.61\)
& \(0.0050\)
& \(1.00\times\) \\
Muon, \(\eta=0.0075\)
& \(0.0561\)
& \(2.3027\)
& \(49.49\)
& \(0.0075\)
& \(1.01\times\) \\
Muon, \(\eta=0.01\)
& \(0.1395\)
& \(2.2052\)
& \(48.98\)
& \(0.0100\)
& \(0.99\times\) \\
\bottomrule
\end{tabular}
\end{table}

\paragraph{DF-Muon cap diagnostic.}
We then ran a matched regularized \(100\)-epoch one-seed diagnostic for DF-Muon, varying only the cap \(\eta_{\max}\). Table~\ref{tab:app_vit_cifar100_df_cap_sweep} reports the result. The cap \(\eta_{\max}=0.01\) gives the best top-1 accuracy among the DF-Muon variants, while remaining close to the tuned fixed-Muon baseline in best validation cross-entropy. We therefore use \(\eta_{\max}=0.01\) in the final three-seed comparison.

\begin{table}[h]
\centering
\caption{
Regularized ViT-Tiny/CIFAR-100 DF-Muon cap diagnostic for one representative
seed. CIFAR-100 images are resized to \(224\times224\). Fixed Muon uses the
tuned learning rate \(\eta=0.001\). DF-Muon varies only the cap
\(\eta_{\max}\). We report best validation cross-entropy and best top-1
accuracy over epochs; runtime is relative to fixed Muon. Since this is a
one-seed diagnostic, no standard deviations are reported.
}
\label{tab:app_vit_cifar100_df_cap_sweep}
\footnotesize
\setlength{\tabcolsep}{3pt}
\resizebox{\linewidth}{!}{%
\begin{tabular}{@{}lccccc@{}}
\toprule
\textbf{Variant}
& \textbf{Best Val. CE}
& \textbf{Top-1 @ Best CE}
& \textbf{Best Top-1}
& \textbf{Mean \(\eta\)}
& \textbf{Rel. time} \\
\midrule
Fixed Muon, \(\eta=0.001\)
& \(\mathbf{1.3508}\)
& \(64.97\)
& \(67.76\)
& \(0.0010\)
& \(1.00\times\) \\
DF-Muon, \(\eta_{\max}=0.005\)
& \(1.3721\)
& \(64.11\)
& \(68.25\)
& \(0.0050\)
& \(1.00\times\) \\
DF-Muon, \(\eta_{\max}=0.01\)
& \(1.3572\)
& \(64.47\)
& \(\mathbf{68.57}\)
& \(0.0100\)
& \(1.00\times\) \\
DF-Muon, \(\eta_{\max}=0.015\)
& \(1.3744\)
& \(63.94\)
& \(68.42\)
& \(0.0150\)
& \(1.00\times\) \\
DF-Muon, \(\eta_{\max}=0.02\)
& \(1.3633\)
& \(64.85\)
& \(67.73\)
& \(0.0200\)
& \(1.04\times\) \\
DF-Muon, \(\eta_{\max}=0.03\)
& \(1.3543\)
& \(64.38\)
& \(67.56\)
& \(0.0297\)
& \(1.03\times\) \\
\bottomrule
\end{tabular}%
}
\end{table}

\begin{figure}[h]
    \centering
    \includegraphics[width=0.92\linewidth]{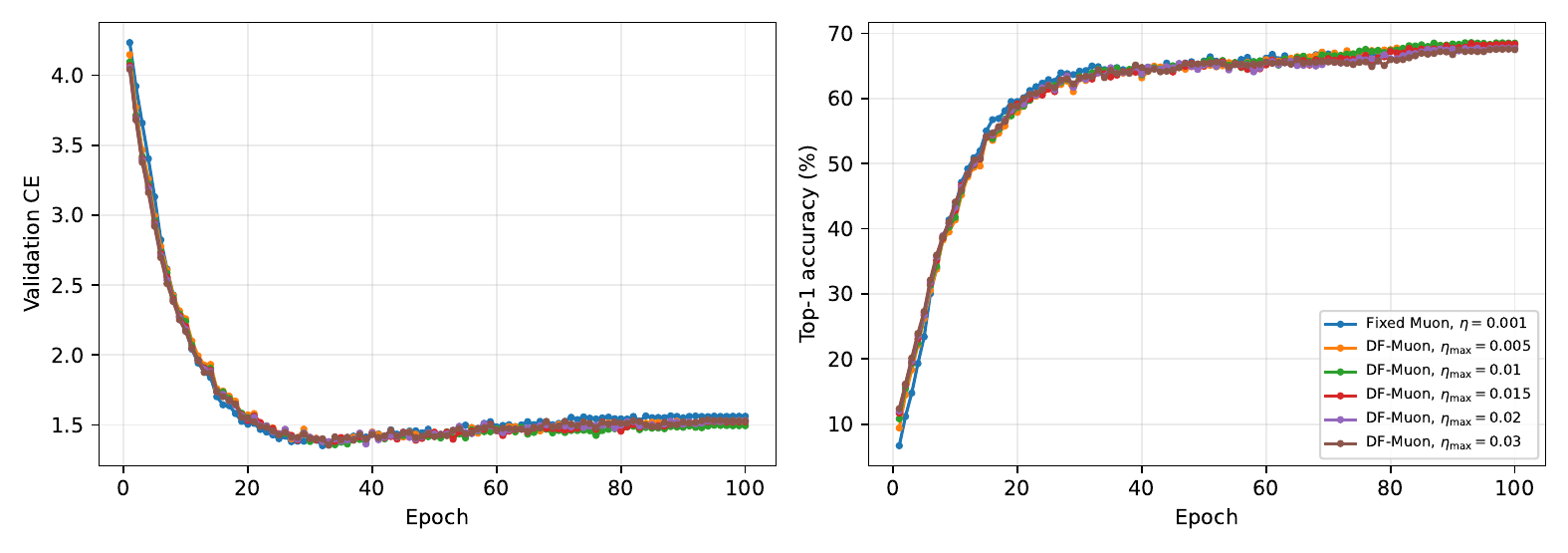}
    \caption{
    DF-Muon cap diagnostic on regularized ViT-Tiny/CIFAR-100 in a
    representative seed. The cap \(\eta_{\max}=0.01\) gives the best top-1
    accuracy among the tested DF-Muon caps, and is used in the final three-seed
    comparison.
    }
    \label{fig:app_vit_cifar100_df_cap_sweep_curves}
\end{figure}

\paragraph{Observed effective scales.}
Because all ViT-Tiny/CIFAR-100 runs use a common warmup--cosine schedule, the schedule-multiplied effective step scale approaches zero near the end of training. We therefore report both the base scale selected by each method and the effective scale after multiplication by the common schedule. Table~\ref{tab:app_vit_cifar100_regularized_eta} summarizes the mean scales over the final \(20\%\) of training.

\begin{table}[h]
\centering
\caption{
Observed base and effective scales in the regularized ViT-Tiny/CIFAR-100 run.
The effective scale equals the base scale selected by each method multiplied by
the common warmup--cosine schedule.
}
\label{tab:app_vit_cifar100_regularized_eta}
\footnotesize
\setlength{\tabcolsep}{5pt}
\begin{tabular}{@{}lcc@{}}
\toprule
\textbf{Method}
& \textbf{Mean base \(\eta\)}
& \textbf{Mean effective scale} \\
\midrule
AdamW
& \(0.0003\)
& \(0.000011\) \\
Best fixed Muon
& \(0.0010\)
& \(0.000036\) \\
DF-Muon
& \(0.0100\)
& \(0.000357\) \\
DA-Muon
& \(0.0100\)
& \(0.000357\) \\
SC-Muon
& \(0.0168\)
& \(0.000594\) \\
\bottomrule
\end{tabular}
\end{table}

\begin{figure}[h]
    \centering
    \includegraphics[width=0.86\linewidth]{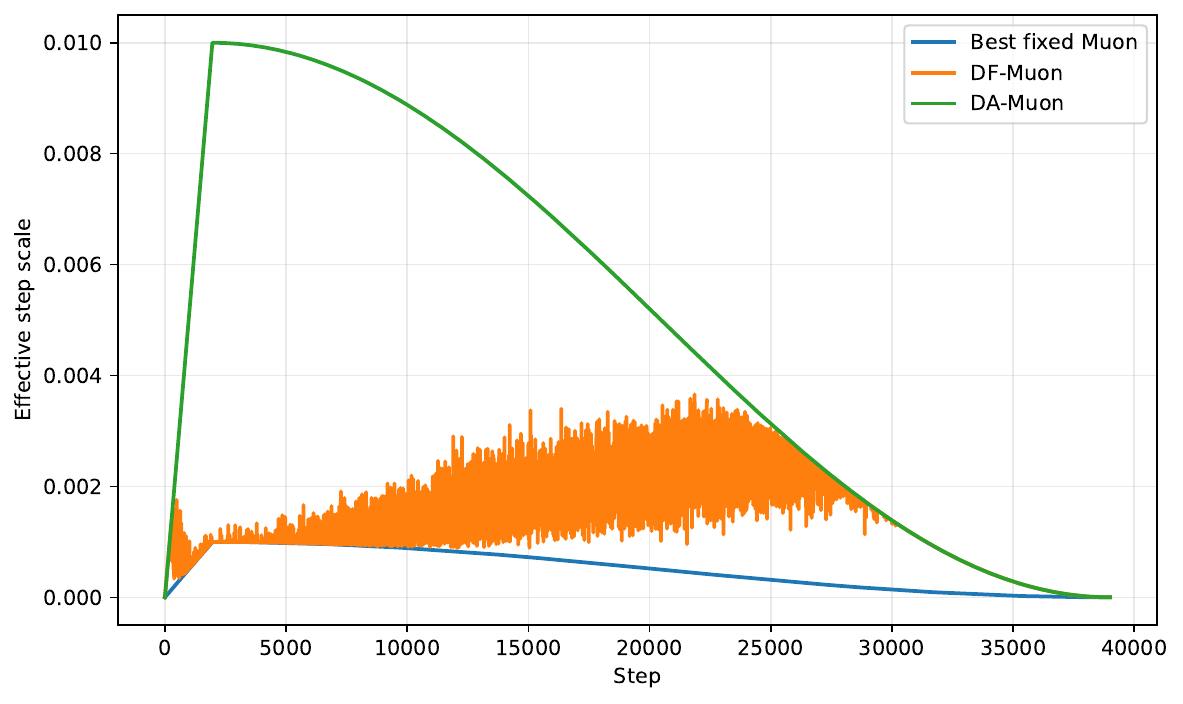}
    \caption{
    Effective step scales on regularized ViT-Tiny/CIFAR-100 for a
    representative seed, excluding SC-Muon for readability. The effective
    scale equals the base scale selected by each method multiplied by the
    common warmup--cosine schedule.
    }
    \label{fig:app_vit_cifar100_effective_scale_no_sc}
\end{figure}

\begin{figure}[h]
    \centering
    \includegraphics[width=0.86\linewidth]{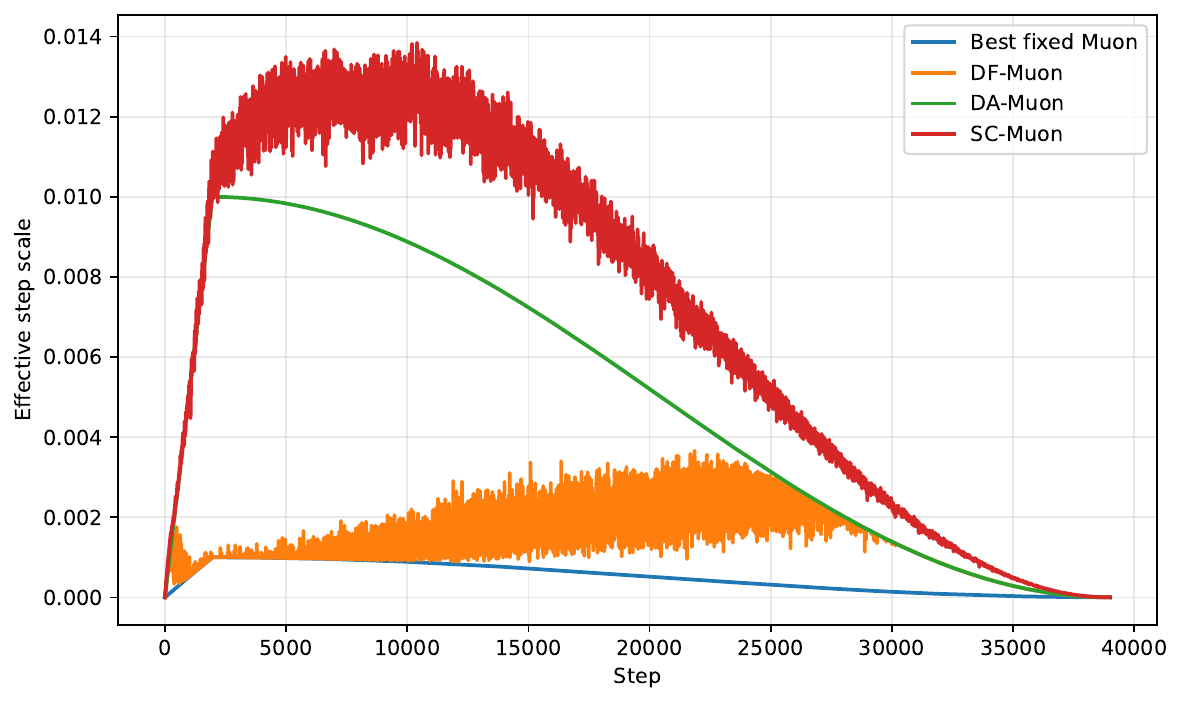}
    \caption{
    Effective step scales on regularized ViT-Tiny/CIFAR-100 for a
    representative seed, including SC-Muon. SC-Muon selects a larger and
    noisier effective scale, while DF-Muon uses an intermediate scale between
    the tuned fixed-Muon baseline and the more aggressive adaptive variants.
    }
    \label{fig:app_vit_cifar100_effective_scale_all}
\end{figure}

\paragraph{Interpretation.}
The diagnostics show that the vision comparison is made against a tuned fixed-Muon baseline and that DF-Muon is not evaluated with an arbitrary cap. The fixed-Muon diagnostic selects \(\eta=0.001\). The DF-Muon cap diagnostic selects \(\eta_{\max}=0.01\), which gives the best top-1 accuracy among the tested caps. In the final three-seed run, fixed Muon obtains the lowest best validation cross-entropy, while DF-Muon obtains the highest best top-1 accuracy with only a small runtime overhead. The scale diagnostics indicate that DF-Muon chooses a more aggressive scale than the fixed-Muon baseline but remains less aggressive than DA-Muon and SC-Muon for much of training.

\subsection{Additional CIFAR-100/ResNet-32 Diagnostics}
\label{app:cifar100_resnet32_diagnostics}

This section reports a complementary image-classification diagnostic on CIFAR-100 with a CIFAR-style ResNet-32. The goal is to check whether the adaptive Muon scalar rules remain competitive with a tuned fixed-Muon baseline outside Transformer language modeling. We train for $T=2000$ optimizer steps with batch size $128$. AdamW uses learning rate $3\cdot 10^{-3}$. Fixed Muon is tuned by a learning-rate sweep, and the adaptive Muon variants use the same maximum base scale as the best fixed-Muon scale.

\paragraph{Fixed-Muon learning-rate sweep.}
Table~\ref{tab:app_cifar100_full_muon_sweep} reports the fixed-Muon learning-rate sweep on one representative seed. The best fixed-Muon learning rate in the sweep is $\eta=0.05$, which achieves the lowest training cross-entropy, lowest test cross-entropy, and highest test accuracy among the tested fixed scales. We therefore use $\eta=0.05$ as the tuned fixed-Muon baseline and set $\eta_{\max}=0.05$ for the adaptive variants.

\begin{table}[h]
\centering
\caption{
Fixed-Muon learning-rate sweep on CIFAR-100/ResNet-32 for one representative
seed.
}
\label{tab:app_cifar100_full_muon_sweep}
\footnotesize
\setlength{\tabcolsep}{6pt}
\begin{tabular}{@{}lccc@{}}
\toprule
\textbf{Method}
& \textbf{Train CE}
& \textbf{Test CE}
& \textbf{Test acc.} \\
\midrule
Muon, $\eta=0.01$
& $1.9777$
& $3.8283$
& $45.93\%$ \\
Muon, $\eta=0.015$
& $1.8756$
& $3.7921$
& $48.03\%$ \\
Muon, $\eta=0.02$
& $1.7952$
& $3.7490$
& $50.16\%$ \\
Muon, $\eta=0.03$
& $1.7495$
& $3.7402$
& $50.34\%$ \\
Muon, $\eta=0.05$
& $\mathbf{1.6675}$
& $\mathbf{3.7171}$
& $\mathbf{52.20\%}$ \\
\bottomrule
\end{tabular}
\end{table}

\begin{figure}[h]
    \centering
    \includegraphics[width=0.82\linewidth]{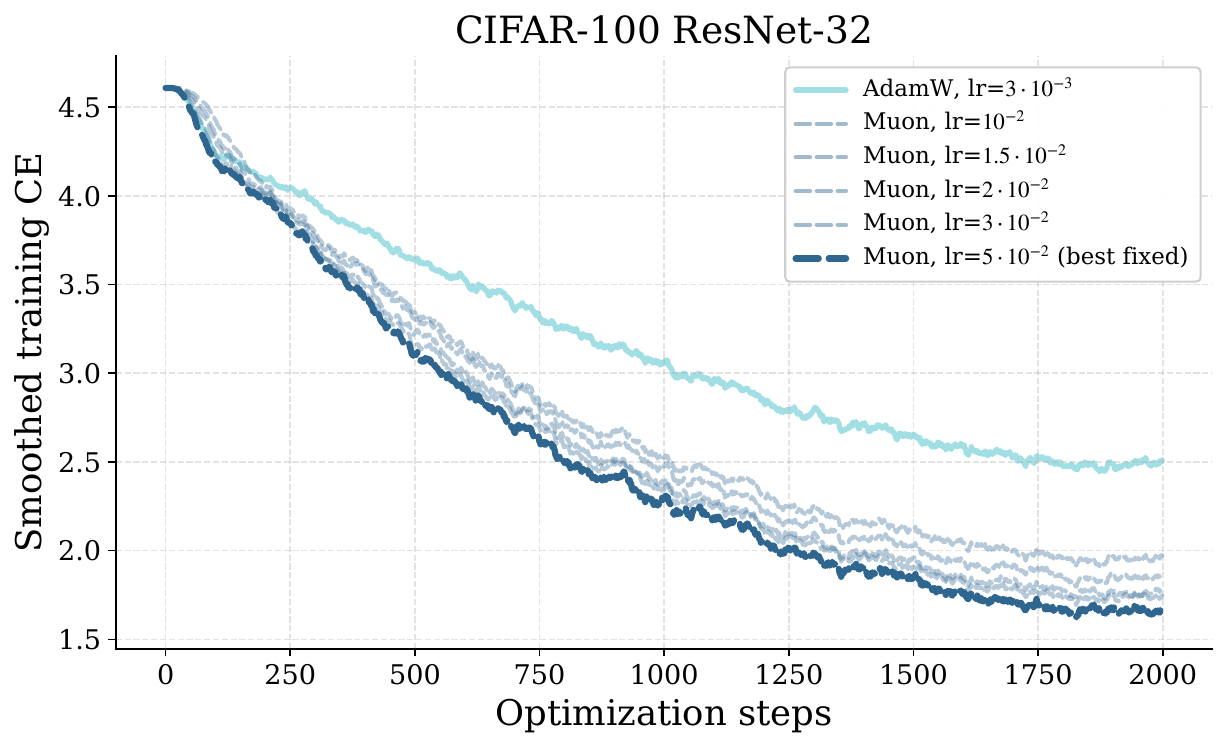}
    \caption{
    Fixed-Muon learning-rate sweep on CIFAR-100/ResNet-32. The largest tested
    scale, $\eta=0.05$, gives the strongest final training curve and is used
    as the tuned fixed-Muon baseline.
    }
    \label{fig:app_cifar100_muon_sweep}
\end{figure}

\paragraph{Three-seed comparison.}
Table~\ref{tab:app_cifar100_resnet32_three_seeds} reports the three-seed aggregate for AdamW, tuned fixed Muon, and the adaptive Muon variants. All Muon-based methods substantially outperform AdamW under this short training budget. The updated majorized DF-Muon improves the tuned fixed-Muon baseline from $52.24\%$ to $52.96\%$ mean test accuracy and slightly improves mean test cross-entropy. DA-Muon gives the lowest mean training CE, while SC-Muon gives the lowest mean test CE.

\begin{table}[h]
\centering
\caption{
CIFAR-100/ResNet-32 results over three seeds. We report mean $\pm$ standard
deviation over seeds $\{42,1337,2024\}$. The fixed-Muon baseline uses the best
learning rate from the sweep, $\eta=0.05$; adaptive Muon variants use
$\eta_{\max}=0.05$.
}
\label{tab:app_cifar100_resnet32_three_seeds}
\footnotesize
\setlength{\tabcolsep}{4pt}
\begin{tabular}{@{}lcccc@{}}
\toprule
\textbf{Method}
& \textbf{Train CE}
& \textbf{Test CE}
& \textbf{Test acc.}
& \textbf{Mean base $\eta$} \\
\midrule
AdamW
& $2.5096 \pm 0.0018$
& $3.9469 \pm 0.1759$
& $34.24 \pm 0.48\%$
& $0.0030$ \\
Best fixed Muon
& $1.6682 \pm 0.0147$
& $3.5589 \pm 0.1477$
& $52.24 \pm 0.20\%$
& $0.0500$ \\
DF-Muon
& $1.6271 \pm 0.0115$
& $3.5505 \pm 0.1834$
& $\mathbf{52.96 \pm 0.54\%}$
& $0.0500$ \\
DA-Muon
& $\mathbf{1.6241 \pm 0.0017}$
& $3.5430 \pm 0.1754$
& $52.91 \pm 0.26\%$
& $0.0500$ \\
SC-Muon
& $1.6321 \pm 0.0114$
& $\mathbf{3.5407 \pm 0.1894}$
& $52.47 \pm 0.65\%$
& $0.0350$ \\
\bottomrule
\end{tabular}
\end{table}

\paragraph{Representative training curve.}
Figure~\ref{fig:app_cifar100_training_curve} shows the representative training curve for seed $1337$. The Muon-based methods reduce training cross-entropy much faster than AdamW. Among the Muon variants, the adaptive methods remain close to the tuned fixed-Muon baseline throughout training.

\begin{figure}[h]
    \centering
    \includegraphics[width=0.82\linewidth]{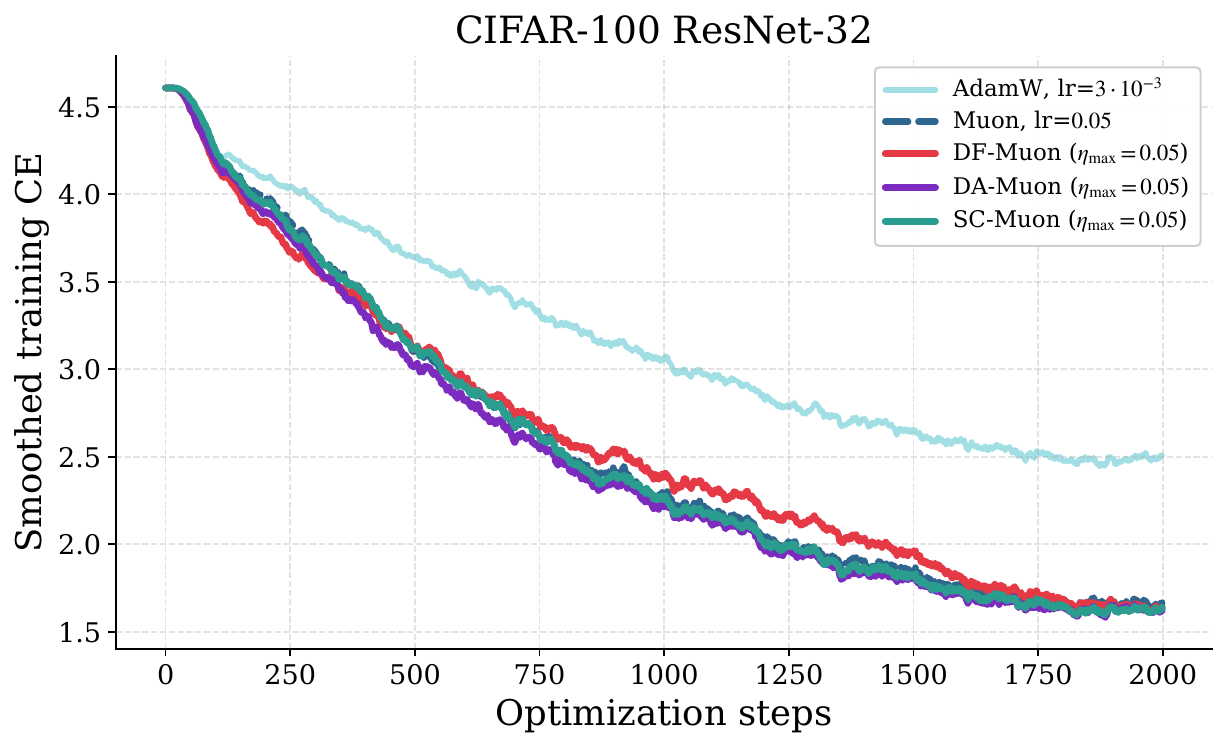}
    \caption{
    CIFAR-100/ResNet-32 training in a representative seed. We compare AdamW,
    the tuned fixed-Muon baseline, and the three adaptive Muon variants using
    the updated majorized DF-Muon implementation.
    }
    \label{fig:app_cifar100_training_curve}
\end{figure}

\paragraph{Observed adaptive scales.}
Because all methods use a warmup--cosine schedule, the schedule-multiplied actual step size approaches zero near the end of training. We therefore report the base scale $\eta_t$ selected by each adaptive rule before multiplying by the common schedule. Table~\ref{tab:app_cifar100_observed_eta} reports the mean base scale over the final $20\%$ of training, averaged over seeds, and Figure~\ref{fig:app_cifar100_observed_eta} shows the representative scale trajectories.

DA-Muon reaches the cap $\eta_{\max}=0.05$ on this task. SC-Muon selects a smaller stable scale around $0.035$. DF-Muon follows a data-dependent trajectory under the majorized scalar rule: it is initially conservative and then increases toward the cap late in training.

\begin{table}[h]
\centering
\caption{
Observed base scales on CIFAR-100/ResNet-32 over three seeds. For adaptive
methods, the mean base scale is averaged over the final $20\%$ of training and
then averaged over seeds.
}
\label{tab:app_cifar100_observed_eta}
\footnotesize
\setlength{\tabcolsep}{5pt}
\begin{tabular}{@{}lcc@{}}
\toprule
\textbf{Method}
& \textbf{Selection rule}
& \textbf{Mean base $\eta$} \\
\midrule
AdamW
& fixed
& $0.0030$ \\
Best fixed Muon
& fixed sweep
& $0.0500$ \\
DF-Muon
& majorized adaptive
& $0.0500$ \\
DA-Muon
& trajectory radius
& $0.0500$ \\
SC-Muon
& certificate
& $0.0350$ \\
\bottomrule
\end{tabular}
\end{table}

\begin{figure}[h]
    \centering
    \includegraphics[width=0.82\linewidth]{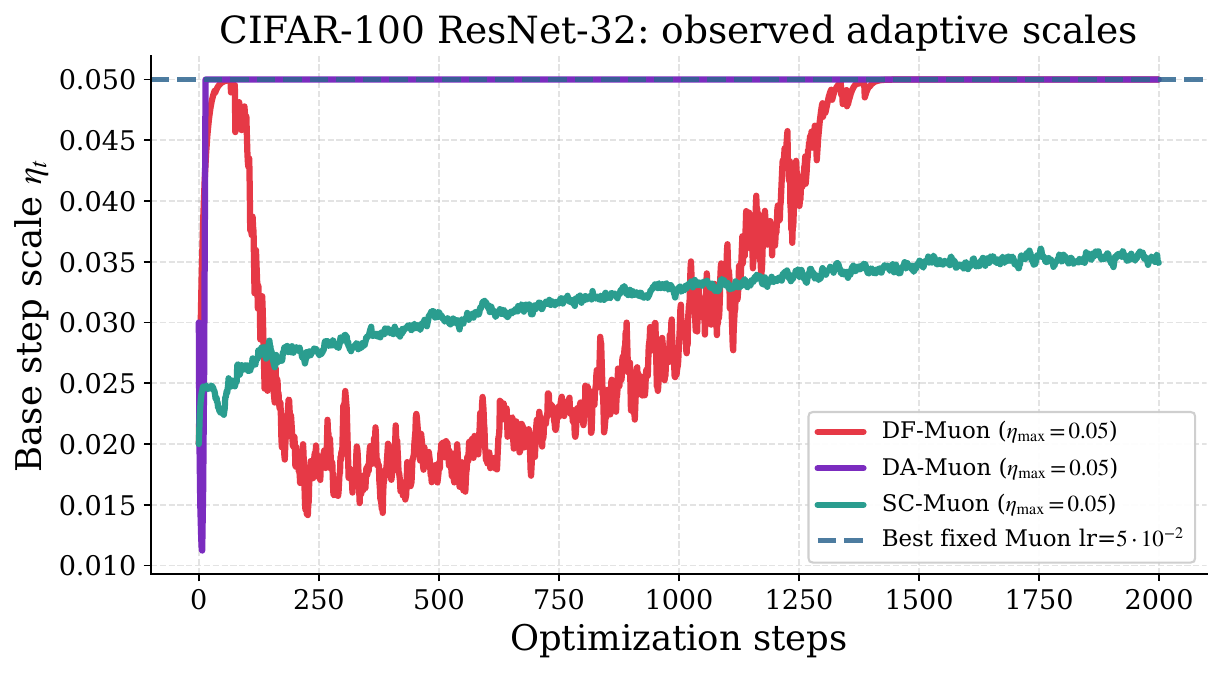}
    \caption{
    Observed base scales selected by the adaptive Muon variants on
    CIFAR-100/ResNet-32 in representative seed $1337$. The dashed horizontal
    line marks the tuned fixed-Muon scale, $\eta=0.05$. DA-Muon reaches the
    cap, SC-Muon stabilizes around $0.035$, and DF-Muon follows a
    conservative-then-increasing trajectory under the majorized scalar rule.
    }
    \label{fig:app_cifar100_observed_eta}
\end{figure}

\paragraph{Interpretation.}
The CIFAR-100/ResNet-32 experiment gives a complementary picture to the GPT-124M/WikiText-103 and NanoGPT/WikiText-2 language-modeling experiments. The fixed-Muon sweep identifies a strong baseline at $\eta=0.05$, and the adaptive variants remain competitive with this tuned baseline. The updated majorized DF-Muon improves mean test accuracy and slightly improves mean test cross-entropy relative to fixed Muon, while DA-Muon and SC-Muon are also competitive. The differences among the Muon variants are small relative to seed-to-seed variation, so we view this experiment as evidence that adaptive scalar scaling remains robust on a convolutional benchmark rather than as a decisive separation among the adaptive rules.

\subsection{Existing Assets, Licenses, and Implementations}
\label{app:assets_licenses}

All experiments use standard public benchmark datasets and do not introduce or redistribute new datasets. For language modeling, we use WikiText-103 in the main GPT-124M experiment and WikiText-2 in the smaller NanoGPT diagnostic experiment. Both are standard WikiText benchmarks derived from Wikipedia text and distributed under Creative Commons Attribution--ShareAlike terms \citep{merity2016pointer}. For image classification, we use CIFAR-100 with its standard train/test split \citep{krizhevsky2009learning}. In all cases, the datasets are used only as benchmarks and are not redistributed as part of this work.

The smaller language-modeling diagnostics use a NanoGPT-style GPT implementation based on the public NanoGPT codebase, which is released under the MIT License. The image-classification experiments use ViT-Tiny and CIFAR-style ResNet-32 implementations in PyTorch; dataset loading and standard image transforms are based on TorchVision \citep{paszke2019pytorch}. PyTorch and TorchVision are released under BSD-style licenses. We modified the training loops only to replace the optimizer on matrix-valued parameters by the Muon-based methods studied in this paper, while keeping the dataset splits, model families, and evaluation metrics standard.
\section{Additional Related Work}
\label{app:additional_related_work}

\paragraph{Broader Muon variants.}
Recent work has studied complementary aspects of Muon-type methods, including matrix-sign computations, variance reduction, error feedback and broader non-Euclidean SGD frameworks~\citep{amsel2025polar,qian2025muon,gruntkowska2025error,kovalev2025non}. These directions are orthogonal to our focus on adaptive scalar scaling.


\end{document}